\documentclass[lettersize,journal]{IEEEtran}
\def\BibTeX{{\rm B\kern-.05em{\sc i\kern-.025em b}\kern-.08em
    T\kern-.1667em\lower.7ex\hbox{E}\kern-.125emX}}
\usepackage{balance}

\usepackage{latexsym,mathrsfs}
\usepackage{amsmath,amssymb} 
\usepackage{amsthm,enumerate,verbatim}
\usepackage{amsfonts}
\usepackage{graphicx}
\usepackage[linesnumbered,ruled,vlined,commentsnumbered]{algorithm2e}
\usepackage{url}
\usepackage{pgfplots}
\usepgfplotslibrary{statistics}
\usepackage[hidelinks]{hyperref}
\usepackage{subcaption}
\usepackage{array}
\usepackage{diagbox}
\usepackage{cleveref}

\usepgfplotslibrary{fillbetween}
\usepgfplotslibrary{groupplots}

\pgfplotsset{compat=1.17}
\newtheorem{lemma}{Lemma}

\newtheorem{theorem}{Theorem}

\newtheorem{definition}{Definition}
\newtheorem*{definition*}{Definition}
\newtheorem{example}{Example}

\newtheorem{remark}{Remark}

\DeclareMathOperator*{\conv}{conv}

\DeclareMathOperator*{\cone}{cone} 
\DeclareMathOperator*{\rank}{rank}

\DeclareMathOperator*{\col}{col}

\Crefname{figure}{Fig.}{Figs.}
\crefname{algorithm}{Alg.}{Algs.}
\Crefname{algorithm}{Algorithm}{Algorithms}

\newcommand{\overbar}[1]{\mkern 1.5mu\overline{\mkern-1.5mu#1\mkern-1.5mu}\mkern 1.5mu}
\definecolor{myblue}{rgb}{0.0, 0.5, 1}

\newcommand{\ovt}[1]{{\color{black}#1}}

\newcommand{\revise}[1]{{{\color{black} #1}}} 

\definecolor{brightpink}{rgb}{1.0, 0.0, 0.5}

\newcommand{\ngi}[1]{{{\color{black} #1}}}

\title{Bounded Simplex-Structured Matrix Factorization: \\ \revise{Algorithms, Identifiability and Applications}} 
\date{}

\author{Olivier {Vu Thanh}, Nicolas Gillis, Fabian Lecron\thanks{University of Mons,  
Rue de Houdain 9, 7000 Mons, Belgium. 
Emails: \{olivier.vuthanh, nicolas.gillis, fabian.lecron\}@umons.ac.be.
The authors acknowledge the support by 
the Fonds de la Recherche Scientifique - FNRS (F.R.S.-FNRS) under the Research Project T.0097.22, 
 by the  
F.R.S.-FNRS and the Fonds Wetenschappelijk Onderzoek - Vlanderen (FWO) under EOS Project no O005318F-RG47, and 
by the Francqui Foundation.}   
	}

\begin{document}

\markboth{Journal of \LaTeX\ Class Files,~Vol.~X, No.~X, X~XXXX}%
{Shell \MakeLowercase{\textit{et al.}}: A Sample Article Using IEEEtran.cls for IEEE Journals}

\IEEEpubid{0000--0000/00\$00.00~\copyright~2022 IEEE}

\maketitle

\begin{abstract}
In this paper, we propose a new low-rank matrix factorization model dubbed bounded simplex-structured matrix factorization (BSSMF). Given an input matrix $X$ and a factorization rank $r$, BSSMF looks 
for a matrix $W$ with $r$ columns and a matrix $H$ with $r$ rows such that $X \approx WH$ where the entries in each column of $W$ are bounded, that is, they belong to given intervals, and the columns of $H$ belong to the probability simplex, that is, $H$ is column stochastic. 
BSSMF generalizes nonnegative matrix factorization (NMF), and simplex-structured matrix factorization (SSMF). 
BSSMF is particularly well suited when the entries of the input matrix $X$ belong to a given interval; for example when the rows of $X$ represent images, or $X$ is a rating matrix such as in the Netflix and MovieLens \revise{datasets} where the entries of $X$ belong to the interval $[1,5]$. 
The simplex-structured matrix $H$ not only leads to an easily understandable decomposition providing a soft clustering of the columns of $X$, 
but implies that the entries of each column of $WH$ belong to the same intervals as the columns of $W$. 
In this paper, we first propose a fast algorithm for BSSMF, even in the presence of missing data in $X$. 
Then we provide identifiability conditions for BSSMF, that is, we provide conditions under which BSSMF admits a unique decomposition, up to trivial ambiguities.  
Finally, we illustrate the effectiveness of BSSMF on two applications: 
extraction of features in a set of images, and 
the matrix completion problem for recommender systems.  
\end{abstract}


\section{Introduction}

\IEEEPARstart{L}{ow-rank} matrix factorizations have recently emerged as very efficient models for unsupervised learning; see, e.g.,~\cite{Vaswani2018PCA, udell2019big} and the references therein.  
The most notable example is principal component analysis (PCA), which can be solved efficiently via the singular value decomposition. In the last 20 years, many new more sophisticated models have been proposed, 
such as 
sparse PCA that requires one of the factors to be sparse to improve interpretability~\cite{dAspremont2007spca}, 
robust PCA to handle gross corruption and outliers~\cite{chandrasekaran2011rank, candes2011robust}, and low-rank matrix completion, also known as PCA with missing data, to handle missing entries in the input matrix~\cite{koren2009matrix}. 

Among such methods, nonnegative matrix factorization (NMF), popularized by Lee and Seung in 1999~\cite{lee1999learning}, required the factors of the decomposition to be component-wise nonnegative. 
More precisely, 
given an input matrix $X \in\mathbb{R}^{m \times n}$ and a factorization rank $r$, NMF looks 
for a nonnegative matrix $W$ with $r$ columns and a nonnegative matrix $H$ with $r$ rows such that $X \approx WH$. NMF has been shown to be useful in many applications, including topic modeling, image analysis, hyperspectral unmixing, and audio source separation; see~\cite{cichocki2009nonnegative, xiao2019uniq, gillis2020} for more examples. 
The main advantage of NMF compared to previously introduced low-rank models is that the nonnegativity constraints on the factors $W$ and $H$ lead to an easily interpretable part-based decomposition. 

More recently, simplex-structured matrix factorization (SSMF) was introduced as a generalization of NMF~\cite{wu2017stochastic}; see also~\cite{abdolali2021simplex} and the references therein. 
SSMF does not impose any constraint on $W$, while it requires $H$ to be column stochastic, that is, $H(:,j) \in \Delta^r$ for all $j$, where  ${\Delta^r = \{ x \revise{ \in\mathbb{R}^r} \ | \ x \geq 0, e^\top x = 1 \}}$ is the probability simplex  and $e$ is the vector of all ones of appropriate dimension.  SSMF is closely related to various machine learning problems, such as latent Dirichlet allocation, clustering, 
and the mixed membership stochastic block model; see~\cite{bakshilearning} and the references therein. 
Let us recall why SSMF is a generalizarion of NMF \revise{by considering the exact NMF model, $X=WH$. Let us normalize the input matrix such that the entries in each column sum to one (w.l.o.g.\ we assume $X$, and $W$, do not have a zero column), that is, such that $X^\top e = e$, and 
 let us impose w.l.o.g.\ that the entries in each column of $W$ also sum to one (by the scaling degree of freedom in the factorization $WH$), that is, $W^\top e = e$. Then, we have }
 \begin{equation} \label{eq:sumtoone}
 X^\top e = e = (WH)^\top e = H^\top W^\top e = H^\top e, 
 \end{equation}
 so that $H$ has to be column stochastic\revise{, since $H \geq 0$ and $H^\top e = e$ is equivalent to $H(:,j)\in\Delta^r$ for all $j$}. 
 
\IEEEpubidadjcol

In this paper, we introduce bounded simplex-structured matrix factorization (BSSMF). 
 BSSMF \revise{imposes the columns of $W$ to belong to a hyperrectangle}
 , namely $W(i,j) \in [a_i, b_i]$ for all $i,j$ for some parameters $a_i \leq b_i$ for all $i$. For simplicity, 
 given $a \leq b \in \mathbb{R}^m$, 
 we denote the hyperrectangle 
 \[
 [a,b] = \{ y \in \mathbb{R}^m 
 \ | \ a_i \leq y_i \leq b_i \text{ for all } i \}, 
 \]
 and refer to it as an interval. \revise{The hyperrectangle constraint on $W$ is denoted as $W(:,j)\in[a,b]$ for all $j$.}  
Let us formally define BSSMF.

\begin{definition}[BSSMF] 
Let $X \in \mathbb{R}^{m \times n}$, let \mbox{$r \leq \min(m,n)$} be an integer, and 
let $a, b \in \mathbb{R}^m$ with $a \leq b$.  
The pair 
$(W,H) 
\in 
\mathbb{R}^{m \times r} 
\times 
\mathbb{R}^{r \times n}$ 
is a BSSMF of $X$ of size $r$ for the interval $[a,b]$ if 
\[
W(:,k) \in [a,b] \text{ for all } k, 
\quad  
H \geq 0, \quad \text{ and } \quad  
H^\top e = e. 
\]
\end{definition}  
\revise{Since the columns of $H$ define convex combinations, the convex hull of the columns of  $X=WH$ is contained in the convex hull of the columns of $W$, which is itself contained in the hyperrectangle $[a,b]$. 
This implies that the hyperrectangle $[a,b]$ must contain the columns of the data matrix, $X=WH$.} 

BSSMF reduces to SSMF 
when $a_i = -\infty$ and $b_i = +\infty$ for all $i$. 
When $X \geq 0$, BSSMF reduces to NMF 
when $a_i = 0$ and $b_i = +\infty$ for all $i$, after a proper normalization of $X$; see the discussion around Equation~\eqref{eq:sumtoone}.

\paragraph{Outline and contribution of the paper}

The paper is organized as follows. 
In Section~\ref{sec_motiv}, we explain the motivation of introducting BSSMF.
In Section~\ref{sec_algo}, we propose an efficient algorithm for BSSMF.
In Section~\ref{sec_identif}, we provide an  identifiability result for BSSMF, 
which follows from an identifiability results for NMF. 
In Section~\ref{sec_numexp}, we illustrate the effectiveness of BSSMF on two applications: \revise{
\begin{itemize}
\item Image feature extraction: the entries of $X$ are pixel intensities. For example, for a gray level image, the entries of $X$ belong to the interval $[0,255]$.

\item Recommender systems: 
the entries of $X$ are ratings of users for some items (e.g., movies). These ratings belong to an interval, e.g., [1,5] for the Netflix and MovieLens \revise{datasets}.

\end{itemize}}

\begin{remark}[Extended conference paper] 
This paper is an extended version of our conference paper~\cite{vuthanh2022bounded}. It provides significant new material: 
\begin{itemize}

    \item A more thorough discussion on the background and the motivations to introduce BSSMF; see Section~\ref{sec_motiv}. 
    
    \item A new algorithm handling missing data 
    and accelerated via data centering; see Section~\ref{sec_algo}. 
    
    \item Illustrations, examples, and proof for the identifiability of BSSMF; see Section~\ref{sec_identifiability_bssmf}. 
    
    \item New numerical experiments on \revise{synthetic, and} the MNIST and MovieLens \revise{datasets}; see Sections~\ref{sec_convspeed} and~\ref{sec_numexp}.  
    
\end{itemize}
\end{remark} 







\section{Motivation of BSSMF} \label{sec_motiv}

The motivation to introduce BSSMF is mostly fourfold; this is described in the next four paragraphs. 

\paragraph{Bounded low-rank approximation}

When the data naturally belong to intervals, imposing the approximation to belong to the same interval allows to provide better approximations, taking into account this prior information. 
\revise{Imposing that the entries in $W$ belong to some interval and that $H$ is column stochastic resolves this issue. BSSMF implies that the columns of the  approximation $WH$ belong to the same interval as the columns of $W$. In fact,  for all $j$, 
\[
X(:,j)  \; \approx \;  W H(:,j)  \; \in \;  [a,b], 
\]
since $W(:,k) \in [a,b]^m$ for all $k$, and the entries of ${H(:,j)}$ are nonnegative and sum to one. }

Another closely related model was proposed in~\cite{liu2021factor} where the entries of the factors $W$ and $H$ are required to belong to bounded intervals.  
The authors showed that their model is suitable for clustering. Nonetheless, it is not clear how to choose the lower and upper bounds on the entries of $W$ and $H$ to obtain tight lower and upper bounds for their product $WH$. \revise{With BSSMF the choice for the lower and upper bounds is easier, e.g., choosing $a_i$ and $b_i$ to be the smallest and largest entry in $X(i,:)$, respectively, that is, bounding $W$ in the same way the data matrix is; see Section~\ref{sec_identifiability_bssmf} for more details.} 

\paragraph{Interpretability} 




\revise{BSSMF allows us to easily interpret both factors: the columns of $W$ can be interpreted in the same way as the columns of $X$ (e.g., as movie ratings, or pixel intensities), while the columns of $H$ provide a soft clustering of the columns of $X$ as they are column stochastic. }
 BSSMF can be interpreted geometrically similarly as SSMF and NMF: the convex hull of the columns of $W$, $\conv(W)$, must contain $\conv(X)$, since $X(:,j) = WH(:,j)$ for all $j$ where $H$ is column stochastic, while it is contained in 
the hyperrectangle $[a,b]$:
\[
\conv(X) 
\quad \subseteq \quad 
\conv(W) 
\quad \subseteq \quad 
\revise{[a,b].}
\] 

\revise{Imposing bounds on the approximation, via the element-wise constraints $a \leq WH \leq b$ for some $a, b \in \mathbb{R}$, was proposed in~\cite{kannan2014bounded} and applied successfully to recommender systems. 
However, this model does not allow to interpret the basis factor, $W$, in the same way as the data. Some elements in $W$ will probably be out the rating range because $W$ is not directly constrained. Hence, the basis elements in $W$ can only be interpreted as ``eigen users'', while with BSSMF, the basis elements can be interpreted as virtual meaningful users. 
It is also difficult to interpret the factor $H$ as it could contain negative contributions. In fact, only imposing $a \leq WH \leq b$ typically leads to dense factors $W$ and $H$ (that is, factors that do not contain many zeros, as opposed to sparse factors), 
while in most applications interpretability usually comes with a certain sparsity degree in at least one of the factors.}

A closely related model \revise{that tackles blind source separation} is bounded component analysis (BCA) proposed in~\cite{cruces2010bounded, erdogan2013class}, where the \revise{sources} are assumed to belong to compact sets (hyperrectangle being a special case), while no constraints is imposed on \revise{the mixing matrix}. Again, without any constraints on \revise{the mixing matrix}, BCA will \revise{generate} dense factors \revise{with negative linear combinations} which are difficult to interpret. 
\revise{Let us note that their motivation is different than ours, as their objective is to extract mixed sources, while ours is to extract interpretable features and decompose data through them. In \cite{mansour2002blind}, the authors also proposed a blind source separation algorithm for bounded sources based on geometrical concepts. The mixtures are assumed to belong to a parallelogram. The proposed separation technique is relies on mapping this parallelogram to a rectangle. Again, their objective is to extract mixed sources. Nonetheless, working with a domain different than a hyperrectangle could be of interest for future work.}

\paragraph{Identifiability} 

Identifiability is key in practice as it allows to recover the ground truth that generated the data; see the discussion 
in~\cite{xiao2019uniq, kueng2021binary} and the references therein.
\revise{A drawback of SSMF is that it is never identifiable, see \Cref{sec_identif_ssmf} for further details. On the counterpart NMF can be identifiable, which is discussed in \Cref{sec_identif_nmf}. Nonetheless, the conditions are not mild. For NMF to be identifiable, it is necessary that the supports of the columns of $W$ (that is, the set of non-zero entries) are not contained in one another (this is called a Sperner family), and similarly for the supports of the rows of $H$; 
see, e.g.,~\cite{moussaoui2005non, laurberg2008theorems}. 
This requires the presence of zeros in each column of $W$ and row of $H$,  which can be a strong condition in some applications. For example, in hyperspectral unmixing, $W$ is typically not sparse because it recovers spectral signatures of constitutive materials which are typically positive. 
Although the conditions for NMF (and SSMF) to be identifiable can be weakened using additional constraints and regularization terms, it then requires hyperparameter tuning procedures. 
In~\cite{tatli2021generalized}, they propose a model where the columns of $H$ belong to a polytope. Using a maximum volume criterion on the convex hull of $H$, their model is identifiable under the condition that the convex hull of $H$ contains the ellipsoid of maximum volume inscribed in the constraining polytope. The use of the maximum volume criterion also requires hyperparameter tuning.}
\revise{In~\cite{cruces2010bounded, erdogan2013class}, the sources are identifiable by optimizing some geometric criterion, 
respectively minimizing a perimeter,  
and maximizing the ratio between the volume of an ellipsoid and the volume of a hyperrectangle.  
These identifiability conditions are not relevant to our model.
} 
As we will see in Section~\ref{sec_identif}, BSSMF is identifiable under relatively mild conditions, 
while it does not require parameter tuning, as opposed to most regularized \revise{structured matrix factorization}
models that are identifiable. \revise{Let us note that it is also possible to formulate identifiable nonlinear matrix approximation models like the bilinear model of~\cite{deville2019separability}, but this is out of the scope of this paper.}

\paragraph{Robustness to overfitting} 

Another drawback of NMF and SSMF is that they are rather sensitive to the choice 
of~$r$. When $r$ is \revise{chosen} too large, these two models are over-parameterized and will \revise{typically} lead to overfitting. 
This is a well-known behaviour that can be addressed with additional regularization terms that need to be fine tuned~\cite{rendle2021revisiting}. As we will see experimentally in \Cref{sec_robust} for matrix completion, 
without any parameter tuning, BSSMF is much more robust to overfitting than NMF and unconstrained matrix factorization. The reason is that the additional bound constraints on $W$ and sum-to-one constraint on $H$ prevents columns of $W$ and of $WH$ \revise{from going} outside the feasible range, \revise{$[a,b]$.}  \revise{In turn, BSSMF will be less sensitive to noise and an overestimation of $r$. For example, an outlier that falls outside the feasible set $[a,b]$ will not pose problems to BSSMF, while it may significantly impact the NMF and SSMF solutions.}


\section{Inertial block-coordinate descent algorithm for BSSMF} \label{sec_algo}

In this paper, we consider the following BSSMF problem 
\begin{equation}
\label{eq:BSSMF}
\begin{split}
\min_{W,H} g(W,H)&:=\frac{1}{2}\| X - WH \|_F^2 \\
\text{ such that } & W(:,k) \in [a,b] \text{ for all } k, \\
& H \geq 0, \text{ and } H^\top e=e, 
\end{split}
\end{equation} 
that uses the squared Frobenius norm to measure the error of the approximation.

\subsection{Proposed algorithm}

Most NMF algorithms rely on block coordinate descent methods, that is, they update a subset of the variables at a time, such \revise{as} the popular multiplicative updates of Lee and Seung~\cite{lee2001algorithms}, the \revise{hierarchical} alternating least squares algorithm~\cite{Cichocki07HALS, GG12}, and a fast gradient based algorithm~\cite{guan2012nenmf}; see, e.g.,~\cite[Chapter 8]{gillis2020} and the references therein for more detail. 
  More recently, an inerTial block majorIzation minimization framework for non-smooth non-convex opTimizAtioN (TITAN) was introduced in~\cite{hien2023inertial} and has been shown to be particularly powerful to solve matrix and tensor factorization problems~\cite{hien2019extrapolNMF, man2021accelerating, vuthanh21}. 

To solve~\eqref{eq:BSSMF}, we therefore apply TITAN which updates one block $W$ or $H$ at a time while fixing the value of the other block. In order to update $W$ (resp.\ $H$), TITAN chooses a block surrogate function for $W$ (resp.\ $H$), embeds an inertial term to this surrogate function and then minimizes the obtained inertial surrogate function. We have $\nabla_W g(W,H)=-(X-WH) H^\top$ which is Lipschitz continuous in $W$ with the Lipschitz constant $\|HH^\top\|$. Similarly, $\nabla_H g(W,H)=-W^\top (X-WH)$ is Lipschitz continuous in $H$ with constant $\|W^\top W\|$. Hence, we choose the Lipschitz gradient surrogate for both $W$ and $H$ and choose the Nesterov-type acceleration as analysed in \cite[Section 4.2.1]{hien2023inertial} and \cite[Remark 4.1]{hien2023inertial}, see also~\cite[Section 6.1]{hien2023inertial} and~\cite{vuthanh21} for similar applications. 

In the case of missing entries in $X$, let us consider the more general model 
\begin{equation}
\label{eq:WBSSMF}
\begin{split}
\min_{W,H} g(W,H) & :=\frac{1}{2}\| M\circ(X - WH) \|_F^2 \\
\text{ such that } & W(:,k) \in [a,b] \text{ for all } k, \\
& H \geq 0, \text{ and } H^\top e=e,    
\end{split}
\end{equation} 
where $\circ$ corresponds to the Hadamard product, and $M$ is a weight matrix which can model missing entries using $M(i,j) = 0$ when  $X(i,j)$ is missing, and $M(i,j) = 1$ otherwise. It can also be used in other contexts;  
see, e.g., \cite{gabriel1979lower, SJ03, gillis2011low}.
TITAN can also be used to solve~\eqref{eq:WBSSMF}, where the gradients are equal to ${\nabla_W g(W,H)=-(M \circ (X-WH)) H^\top}$ and ${\nabla_H g(W,H)=-W^\top (M\circ  (X-WH))}$. We acknowledge that the identifiability result that will be presented in Section~\ref{sec_identif} does not hold for the case where some data are missing, this is an interesting direction of future research. 
\cref{alg:BSSMF} describes TITAN for solving the general problem~\eqref{eq:WBSSMF}, where $[.]^a_b$ is the column-wise projection on \revise{$[a,b]$ }
and $[.]_{\Delta^r}$ is the column wise projection on the simplex $\Delta^r$.
When some data \revise{are} missing, the Lipschitz constant of the gradients relatively to $W$ and $H$ could be smaller than $\|HH^\top\|$ and $\|W^\top W\|$, respectively. 
Relatively to $H$ for instance, a smaller Lipschitz constant would be $\max_i\|W^\top (M(:,i)e^\top) \circ W\|$. We arbitrarily choose to keep $\|HH^\top\|$ and $\|W^\top W\|$ even when some data \revise{are} missing since those values are faster to compute.
\revise{Due to our derived algorithm being a particular instance of TITAN with Lipschitz gradient surrogates~\cite[Section 4.2]{hien2023inertial}}, \cref{alg:BSSMF} guarantees a subsequential convergence, that is, every limit point of the generated sequence is a stationary point of Problem~\eqref{eq:BSSMF}. The Julia code for \cref{alg:BSSMF} is available on gitlab\footnote{\href{https://gitlab.com/vuthanho/bssmf.jl}{https://gitlab.com/vuthanho/bssmf.jl}} (a MATLAB code is also available on gitlab\footnote{\href{https://gitlab.com/vuthanho/bounded-simplex-structured-matrix-factorization}{https://gitlab.com/vuthanho/bounded-simplex-structured-matrix-factorization}} but it does not handle missing data for now). \revise{We omit the implementation details here, but let us mention that when data \revise{are} missing, our Julia implementation does not compute the Hadamard product with $M$ explicitly but rather takes advantage of the sparsity of the data by using multithreading to improve the computational time. The projections $[.]^a_b$ and $[.]_{\Delta^r}$ are also computed using multithreading.}

\begin{algorithm}[htb!]
\caption{BSSMF}
\label{alg:BSSMF}
\SetKwInOut{Input}{input}
\SetKwInOut{Output}{output}
\DontPrintSemicolon
\Input{Input data matrix $X \in \mathbb{R}^{m \times n}$, bounds $a\leq b\in\mathbb{R}^m$, initial factors $W \in \mathbb{R}^{m \times r}$ s.t. $W(:,k) \in [a,b]$ for all $k$ and simplex structured $H \in \mathbb{R}^{r \times n}_+$, weights $M\in[0,1]^{m\times n}$}
\Output{$W$ and $H$}
$\alpha_1=1$, $\alpha_2=1$, $W_{old}=W, H_{old}= H$, $L_W^{prev}=L_W=\| H H^\top\|_2$,  $L_H^{prev}=L_H=\|W^\top W\|_2$\;
\Repeat{some stopping criteria is satisfied\nllabel{alg:BSSMF:line:outerloop}}{
    \While{stopping criteria not satisfied\nllabel{alg:BSSMF:line:Winnerloop}}{
        $\alpha_{0}=\alpha_1, \alpha_1=(1+\sqrt{1+4\alpha_0^2})/2$\;
        $\beta_{W}=\min\left[~(\alpha_0-1)/\alpha_{1},0.9999\sqrt{L_W^{prev}/L_W} ~\right] $\;
        $\overbar{W}\leftarrow W+\beta_{W}(W-W_{old}) $\;
        $W_{old}\leftarrow W$\;
        $W \leftarrow \left[\overbar{W}+\frac{(M\circ(X-\overbar{W}H))H^{\top}}{L_W}\right]_a^b$\;
        $ L_W^{prev} = L_W$\;
    }
    $ L_H \leftarrow \|W^\top W\|_2 $\;
    \While{stopping criteria not satisfied\nllabel{alg:BSSMF:line:Hinnerloop}}{
        $\alpha_{0}=\alpha_2, \alpha_2=(1+\sqrt{1+4\alpha_0^2})/2$\;
    	$ \beta_{H}=\min\left[~(\alpha_0-1)/\alpha_{2},0.9999\sqrt{L_H^{prev}/L_H} ~\right] $\;
    	$ \overbar{H} \leftarrow H+\beta_{H}(H-H_{old}) $\;
    	$H_{old}\leftarrow H$\;
    	$ H \leftarrow \left[\overbar{H}+\frac{W^{\top} (M\circ(X-W\overbar{H}))}{L_H} \right]_{\Delta^r}$\;\nllabel{alg:BSSMF:line:proj}
    	$L_H^{prev} \leftarrow L_H$\;
    }
    $ L_W = \|HH^\top\|_2 $\;
}
\end{algorithm}

A simple choice to initialize the factors, $W$ and $H$, in \cref{alg:BSSMF} is to randomly initialize them: for all $i$, each entry of $W(i,:)$ is generated using the uniform distribution in the interval $[a_i,b_i]$, while $H$ is generated using a uniform distribution in $[0,1]^{r\times n}$ whose columns are then projected on the simplex $\Delta^r$.


\subsection{Accelerating BSSMF algorithms via data centering}  \label{sec_centering}

\revise{Not only t}he BSSMF model is invariant to translations of the input data (this is explained in details in \Cref{sec_identifiability_bssmf}), 
\revise{but also the optimization}, because of the simplex constraints. In particular, for any $c\in\mathbb{R}$ and denoting 
\revise{$J$} the matrix of all ones of appropriate dimension, minimizing
\begin{equation}
\label{eq:SSMF}
f(W,H):=\frac{1}{2}\| X - WH \|_F^2
\end{equation}
or 
\begin{equation}
\label{eq:TSSMF}
f_c(W,H):=\frac{1}{2}\| X-cJ - (W-cJ)H \|_F^2, 
\end{equation} 
is equivalent in BSSMF, since $cJ = cJH$ 
as 
$H$ is column stochastic. 
However, \emph{outside the feasible set, $f$ and $f_c$ do not have the same topology}. 
Computing the gradients, we have $\nabla_H f(W,H)=W^\top(WH-X)$ which is Lipschitz continuous in $H$ with the Lipschitz constant $\|W^\top W\|$, and
$\nabla_H f_c(W,H)=W_c^\top(W_cH-X_c)$ which is Lipschitz continuous in $H$ with the Lipschitz constant $\|W_c^\top W_c\|$, where $W_c=W-cJ$ and $X_c=X-cJ$. 
Particularly, for BSSMF, since $W$ can be interpreted in the same way as $X$, we can expect $\operatorname{mean}(X)\approx\operatorname{mean}(W)$, where $\operatorname{mean(.)}$ is the empirical mean of the entries of the input. Consequently, if we choose $c=\operatorname{mean}(X)$, we expect\footnote{We empirically noticed that, very often, with $c=\operatorname{mean}(W)$, $\|W_c^\top W_c\|$ is of the same order of $\min_{c}\|W_c^\top W_c\|$. 
} the Lipschitz constant $\|W_c^\top W_c\|$ to be smaller than $\|W^\top W\|$. A smaller Lipschitz constant means that, when updating $H$, the gradient steps are allowed to be larger without losing any convergence guarantee. Hence, with the right translation on our data $X$, the optimization problem on $H$ is unchanged on the feasible set but \cref{alg:BSSMF} can be accelerated. 

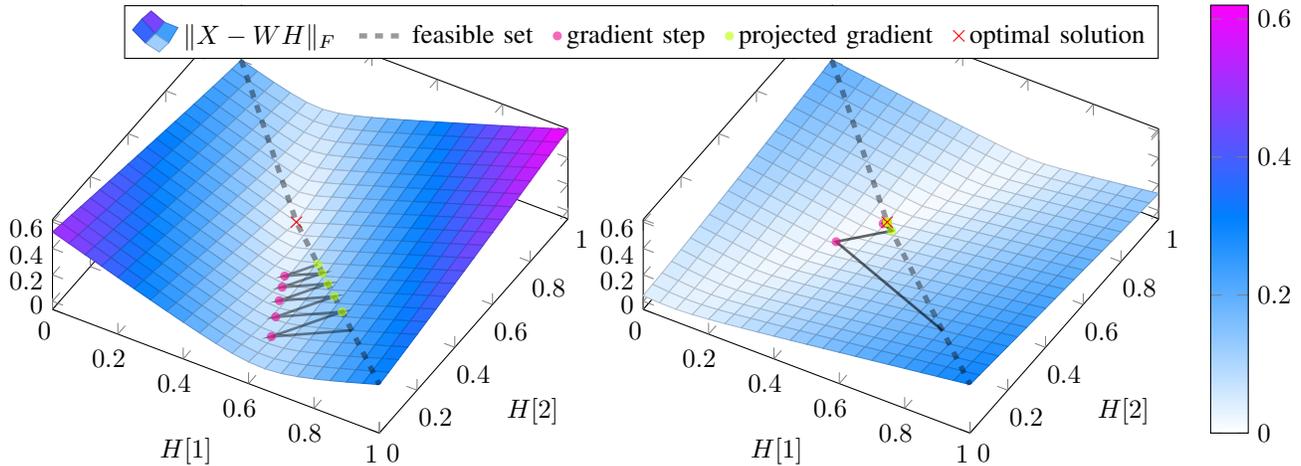
\begin{figure*}[htb!]
    \centering
    \begin{tikzpicture}
 
\begin{groupplot}[view={30}{70}, group style={group size=2 by 1}, colormap/cool, point meta min=0, point meta max = 0.62, zmax=0.62, xlabel={$H[1]$}, ylabel={$H[2]$}, legend columns = -1,legend style={at={(1,1)},/tikz/every even column/.append style={column sep=0.2cm}}]

\nextgroupplot[]
\addplot3 [surf,mesh/ordering=y varies,opacity=1] table{data/error_notcentered.txt};
\addplot3 [line width = 2pt,opacity=0.4,dashed] table{data/feasibleset.txt};
\addplot3 [line width = 1pt,opacity=0.46] table{data/gradientstep_notcentered.txt};
\addplot3 [only marks, color = magenta,opacity=0.6,mark size = 1.5pt] table{data/gradientonly_notcentered.txt};
\addplot3 [only marks, color = lime,opacity=0.6,mark size = 1.5pt] table{data/projonly_notcentered.txt};
\addplot3 [only marks, mark = x,mark size = 3pt,color = red] coordinates {(0.4,0.6,0)};

\nextgroupplot[colorbar]
\addplot3 [surf,mesh/ordering=y varies] table{data/error_centered.txt};
\addlegendentry{$\|X-WH\|_F$}
\addplot3 [line width = 2pt,opacity=0.4,dashed] table{data/feasibleset.txt};
\addlegendentry{feasible set}
\addplot3 [line width = 1pt,opacity=0.6,forget plot] table{data/gradientstep_centered.txt};
\addplot3 [only marks, color = magenta,opacity=0.6,mark size = 1.5pt] table{data/gradientonly_centered.txt};
\addlegendentry{gradient step}
\addplot3 [only marks, color = lime,opacity=0.6,mark size = 1.5pt] table{data/projonly_centered.txt};
\addlegendentry{projected gradient}
\addplot3 [only marks, mark = x,mark size = 3pt,color = red] coordinates {(0.4,0.6,0)};
\addlegendentry{optimal solution}
\end{groupplot}
 
\end{tikzpicture}
    \caption{Influence of centering the data on the cost function topology regarding $H$ via a small example ($m=2,r=2,n=1$). Left: without centering. Right: with centering. 
    Five projected gradient steps are shown, decomposed through one gradient descent step followed by its projection onto the feasible set.}
    \label{fig:centering_topo}
\end{figure*}


Let us illustrate this behavior on a small example with \mbox{$m=2$}, $n=1$, $r=2$. We choose $$
{X=\begin{pmatrix}0.4 & 0.3\\0.7 & 0.2\end{pmatrix}\begin{pmatrix}0.4\\0.6\end{pmatrix}}.
$$ 
We fix $$
W=\begin{pmatrix}0.4 & 0.3\\0.7 & 0.2\end{pmatrix}, 
$$ 
and try to solve\revise{, with respect to $H$,} \cref{eq:SSMF} and \cref{eq:TSSMF}  with $c=\operatorname{mean}(X)$. 
We perform 5 projected gradient steps and display the results on \Cref{fig:centering_topo}. On the left, 5 projected gradient steps are performed to update $H$ based on the original data $X$. On the right, 5 projected gradient steps are performed to update $H$ based on the centered data $X$. 
The feasible sets (in dash) are exactly the same, and therefore the optimal solutions are also the same. 
However, we observe that the landscape of the cost function outside the feasible region is smoother when the data \revise{are} centered. This allows the solver to converge faster towards the optimal solution, as the gradients point better towards the optimal solution and the stepsizes are larger. 
The improvement regarding the convergence speed by applying centering with real data will probably not be as drastic as in this small example. Still, minimizing a smoother function is always advantageous, and this will be shown empirically on real data in \Cref{sec_convspeed}.

\subsection{Convergence speed and effect of acceleration strategies on real data} \label{sec_convspeed}

In this subsection, the goal is twofold: 
(1)~show the effect of the extrapolation in TITAN by comparing \cref{alg:BSSMF} to a non-extrapolated block coordinate descent, and 
(2)~show the acceleration effect of centering the data.

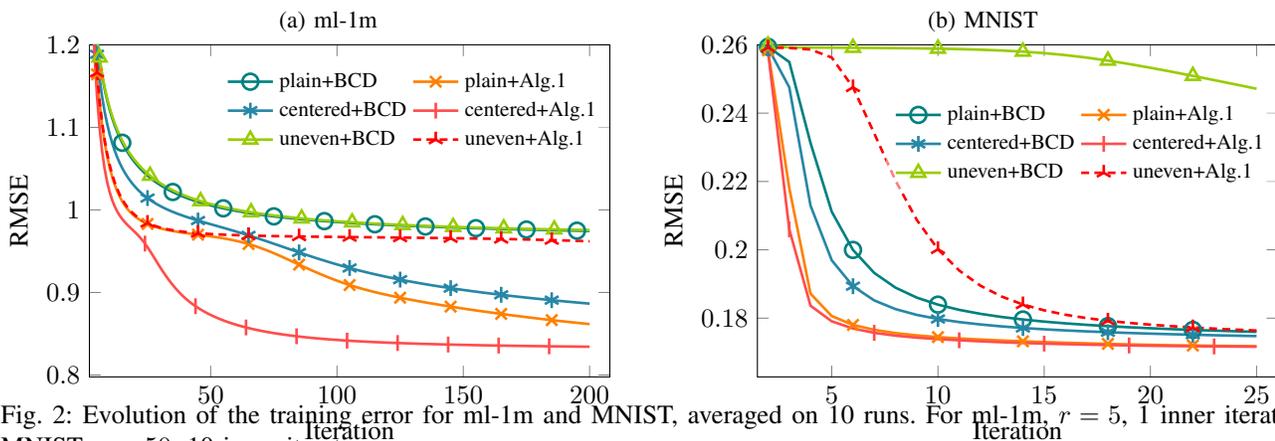
\begin{figure*}[htb!]
    \begin{subfigure}{0.48\textwidth}
        \caption{ml-1m}
        \label{fig:ml1m_extra_center}
        \vspace{-2ex}
        \begin{tikzpicture}
\begin{axis}[
            width=8.5cm,
            height=6cm,
            ymax = 1.2,
            xmin = 2,
            xmax = 208,
            ylabel = {RMSE},
            xlabel = {Iteration},
            cycle list name=exotic,
            legend columns = 2,
            mark size = 3pt,
            mark repeat = 20,
            legend cell align={left},
            legend style={font=\footnotesize,at={(1,0.95)},anchor=north east,draw=none,fill opacity=0.5,text opacity=1}]

\addplot+[mark phase=10,mark = o,line width = 1pt] table[x index = 0, y index = 1]{data/ml1m_extra_center.txt};
\addlegendentry{plain+BCD}
\addplot+[mark = x,line width = 1pt] table[x index = 0, y index = 2]{data/ml1m_extra_center.txt};
\addlegendentry{plain+Alg.\ref{alg:BSSMF}}
\addplot+[mark = asterisk,line width = 1pt] table[x index = 0, y index = 3]{data/ml1m_extra_center.txt};
\addlegendentry{centered+BCD}
\addplot+[mark = |,line width = 1pt] table[x index = 0, y index = 4]{data/ml1m_extra_center.txt};
\addlegendentry{centered+Alg.\ref{alg:BSSMF}}
\addplot+[mark=triangle,line width = 1pt] table[x index = 0, y index = 5]{data/ml1m_extra_center.txt};
\addlegendentry{uneven+BCD}
\addplot+[mark=Mercedes star,line width = 1pt] table[x index = 0, y index = 6]{data/ml1m_extra_center.txt};
\addlegendentry{uneven+Alg.\ref{alg:BSSMF}}

\end{axis}
\end{tikzpicture}
    \end{subfigure}
    \begin{subfigure}{0.48\textwidth}
        \caption{MNIST}
        \label{fig:mnist_extra_center}
        \vspace{-2ex}
        \begin{tikzpicture}
\begin{axis}[
            width=8.5cm,
            height=6cm,
            ymax = 0.26,
            xmin = 1.5,
            xmax = 26,
            ylabel = {RMSE},
            xlabel = {Iteration},
            cycle list name=exotic,
            mark size = 3pt,
            mark repeat = 4,
            legend columns = 2,
            legend cell align={left},
            legend style={font=\footnotesize,at={(1,0.85)},anchor=north east,draw=none,fill opacity=0.5,text opacity=1}]

\addplot+[mark = o,line width = 1pt] table[x index = 0, y index = 1]{data/mnist_extra_center.txt};
\addlegendentry{plain+BCD}
\addplot+[mark = x,line width = 1pt] table[x index = 0, y index = 2]{data/mnist_extra_center.txt};
\addlegendentry{plain+Alg.\ref{alg:BSSMF}}
\addplot+[ mark = asterisk,line width = 1pt] table[x index = 0, y index = 3]{data/mnist_extra_center.txt};
\addlegendentry{centered+BCD}
\addplot+[mark phase=2, mark = |,line width = 1pt] table[x index = 0, y index = 4]{data/mnist_extra_center.txt};
\addlegendentry{centered+Alg.\ref{alg:BSSMF}}
\addplot+[mark = triangle,line width = 1pt] table[x index = 0, y index = 5]{data/mnist_extra_center.txt};
\addlegendentry{uneven+BCD}
\addplot+[mark = Mercedes star,line width = 1pt] table[x index = 0, y index = 6]{data/mnist_extra_center.txt};
\addlegendentry{uneven+Alg.\ref{alg:BSSMF}}
\end{axis}
\end{tikzpicture}
    \end{subfigure}\\[-8ex]
    \caption{Evolution of the training error for ml-1m and MNIST, averaged on 10 runs. For ml-1m, $r=5$, 1 inner iteration. For MNIST, $r=50$, 10 inner iterations.}
    \label{fig:extra_center}
\end{figure*} 
\revise{We will apply the} BSSMF \revise{model on} MNIST and ml-1m \revise{(}these two datasets are properly introduced respectively in \Cref{sec_inter} and \Cref{sec_robust}\revise{)} in six different scenarios: 3 data related scenarios $\times$ 2 algorithmic related scenarios. The data scenarios are raw data, centered data, and data to which a positive offset is added (respectively called `plain', `centered' and `uneven' in \Cref{fig:extra_center}). Note that `uneven', that is, adding a positive constant, will worsen the landscape, as $\|W_c^\top W_c\|$ will increase (since $a \geq 0$). This scenario will allow to further validate our observation about the acceleration effect of preprocessing. 
For each data case, 2 algorithms are tested: 
(1)~\cref{alg:BSSMF}, and (2)~a standard block coordinate descent (BCD) which is \cref{alg:BSSMF} where the $\beta$'s are fixed to 0; this corresponds to the popular proximal alternating linearized minimization (PALM) algorithm~\cite{bolte2014proximal}. 
When the algorithms are compared on the same data scenario, \cref{alg:BSSMF} always converges faster and to a better solution than BCD. We also observe that when the data \revise{are} centered, applying the same algorithm always lead to faster convergence than on the plain case. On ml-1m (\Cref{fig:ml1m_extra_center}), applying BCD on the centered data is almost as fast as applying \cref{alg:BSSMF} in the plain case, meaning that a good preprocessing is almost as important as a good acceleration strategy. If we look at the \revise{root-mean-square error (RMSE)} on the test set, centered data \revise{are} even more important than a good acceleration strategy. Actually, at the end of the experiment from \Cref{fig:ml1m_extra_center}, on the test set, centered+BCD has \revise{an} RMSE of 0.91 while plain+\cref{alg:BSSMF} has \revise{an} RMSE of 0.94. Still on ml-1m, \cref{alg:BSSMF} benefits remarkably from centering the data in comparison to the plain case. We also note an improvement on MNIST (\Cref{fig:mnist_extra_center}) which is not as noticeable as on ml-1m (probably because this problem is less difficult to solve, which is corroborated by the number of iterations required to converge). Here, the uneven case is just shown to remind that the data points cloud position is of much concern and should not be neglected. Globally, regardless of the dataset, applying \cref{alg:BSSMF} on centered data is the best strategy \revise{as compared with using plain data}. As a consequence, it will be our default choice for the experiments in \Cref{sec_numexp}. 

\revise{As mentioned above, when entries are missing, \cref{alg:BSSMF} can take advantage of the sparsity of the data and uses multithreading. We report in~\Cref{tab:time_ml-1m} the computation time of~\cref{alg:BSSMF} in the experiment settings of~\Cref{fig:ml1m_extra_center}, given by the macro \texttt{@btime} from the package \texttt{BenchmarkTools.jl}. When the dataset is full, like with MNIST, multithreading is only used for the projections $[.]^a_b$ and $[.]_{\Delta^r}$. The computation time in the settings of the experiment in~\Cref{fig:mnist_extra_center} is reported in~\Cref{tab:time_mnist}. Note that there is no distinction between~\cref{alg:BSSMF} and BCD in terms of computation time because the computation of the acceleration is negligible compared to the other computations.
\begin{table}[ht]
    \centering
    \begin{tabular}{l|c|c|c|c|c|c|c}
    \# threads & 1 & 2 & 4 & 6 & 8 & 10 & 12\\ \hline
    time (s) & 30.53 & 5.14 & 2.98 & 2.85 & 3.00 & 2.78 & 3.31 
    \end{tabular}
    \caption{Computation time of~\cref{alg:BSSMF} in the experiment settings of~\Cref{fig:ml1m_extra_center} depending on the number of used threads.}
    \label{tab:time_ml-1m}
\end{table}
\begin{table}[ht]
    \centering
    \begin{tabular}{l|c|c|c|c|c|c|c}
    \# threads & 1 & 2 & 4 & 6 & 8 & 10 & 12\\ \hline
    time (s) & 27.79 & 21.92 & 16.67 & 15.22 & 15.73 & 16.01 & 16.65 
    \end{tabular}
    \caption{Computation time of~\cref{alg:BSSMF} in the experiment settings of~\Cref{fig:mnist_extra_center} depending on the number of used threads.}
    \label{tab:time_mnist}
\end{table}
}


\section{Identifiability} \label{sec_identif} 

Let us first define a factorization model. 
\begin{definition}[Factorization model] 
Given a matrix $X \in \mathbb{R}^{m \times n}$, and an integer $r \leq \min(m,n)$, a  factorization model is an optimization model of the form 
\begin{equation} \label{eq:factomodel}
\begin{split}
\min_{W \in \mathbb{R}^{m \times r}, H \in \mathbb{R}^{r \times n}} & 
g(W,H) \\
\text{ such that } & X = WH, \\
& W \in \Omega_W \text{ and } H \in \Omega_H,     
\end{split}
\end{equation}
where $g(W,H)$ is some criterion, and $\Omega_W$ and $\Omega_H$ are the feasible sets for $W$ and $H$, respectively. 
\end{definition}

Let us define the identifiability of a factorization model, and essential uniqueness of a pair $(W,H)$.  

\begin{definition}[Identifiability / Essential uniqueness]
\label{def:identifiability}
Let $X \in \mathbb{R}^{m \times n}$, and $r \leq \min(m,n)$ be an integer. 
Let $(W,H)$ be a solution to a given factorization model~\eqref{eq:factomodel}. 
The pair $(W,H)$ is essentially unique 
for the factorization model~\eqref{eq:factomodel} of matrix $X$ 
if and only if any other pair $(W',H') \in \mathbb{R}^{m \times r} \times \mathbb{R}^{r \times n}$ that solves the factorization model~\eqref{eq:factomodel} satisfies, for all $k$, 
$$W'(:,k) = \alpha_k W(:,\pi(k)) $$
and 
$$H'(k,:) = \alpha_k^{-1} H(\pi(k),:), $$
where $\pi$ is a permutation of $\{1,2,\dots,r\}$, and $\alpha_k \revise{\neq} 0$ for all $k$. 
In other terms, $(W',H')$ can only be obtained as a permutation and scaling of $(W,H)$. 
In that case, the factorization model is said to be identifiable for the matrix~$X$.  
\end{definition} 

A key question in theory and practice is to determine conditions on $X$, $g$, $ \Omega_W$ and $\Omega_H$ that lead to identifiable factorization models; see, e.g., \cite{xiao2019uniq, kueng2021binary} for discussions.  

In the next three sections, we discuss the identifiability of SSMF, NMF, and BSSMF.

\subsection{Simplex-structured matrix factorization (SSMF)} 
\label{sec_identif_ssmf}
Without further requirements, SSMF is never identifiable; which follows from a result for semi-NMF which is a factorization model \revise{that} requires only one factor, $H$, to be nonnegative~\cite{gillis2015exact}.  
Let $X = WH$ be an SSMF of $X$. We can obtain other SSMF of $X$ using the following transformation: for any $\alpha \geq 0$, let 
\[
W(\alpha) :=  
W \left( (1+\alpha) I - \revise{\frac{\alpha}{r}J} \right), 
\]
and
\[ 
\begin{split}
H(\alpha) 
:=& 
\left( 
 \frac{1}{1+\alpha} H 
 + \revise{\frac{\alpha}{(1+\alpha)r} J} 
\right) \\
=& 
\left( 
\frac{1}{1+\alpha} I + \revise{\frac{\alpha}{(1+\alpha)r} J}
\right) 
  H ,   
\end{split}
\]
where $I$ is the identity matrix of appropriate dimension, and the last equality follows from $e^\top H = e^\top$. 
The matrix $H(\alpha)$ is column stochastic since $H$ and $\frac{\revise{J}}{r}$ are. One can check that $(W(\alpha), H(\alpha))$ is not a permutation and scaling of $(W,H)$ for $\alpha > 0$, while $WH = W(\alpha) H(\alpha)$ since\footnote{This is an invertible M-matrix, with positive diagonal elements and negative off-diagonal elements, whose inverse is nonnegative~\cite{berman1994nonnegative}.}  
\[
\begin{split}
A(\alpha)  
 := & \left( 
(1+\alpha) I - \frac{\alpha}{r} \revise{J} \right)^{-1} \\
= &
\frac{1}{1+\alpha} I + \revise{\frac{\alpha}{(1+\alpha)r} J} . 
\end{split}
\] 
Geometrically, to obtain $W(\alpha)$, the columns of $W$ are moved towards the exterior of $\conv(W)$ and hence the convex hull of the column of $W(\alpha)$ contains the convex hull of the columns of $W$ and hence contains $\conv(X)$. This follows from the fact that  $W  = W(\alpha) A(\alpha)$, 
where $A$ is column stochastic.

To obtain identifiability of SSMF, one needs  either to impose additional constraint on $W$ and/or $H$ such as sparsity~\cite{abdolali2021simplex}, 
or look for a solution minimizing a certain function $g$. In particular, the solution $(W,H)$ that minimizes the volume of the convex hull of the columns of $W$ (see Theorem~\ref{th_identifminvolSSMF} below for a formula) is essentially unique given that $H$ satisfies the so-called sufficiently scattered condition (SSC). The SSC is defined as follows. 
\begin{definition}[Sufficiently scattered condition]  \label{def:ssc} 
The matrix $H \in \mathbb{R}^{r \times n}_+$ is sufficiently scattered if the following two conditions are satisfied:\index{sufficiently scattered condition!definition} \vspace{0.1cm}

\noindent [SSC1]  $\mathcal{C} = \{x \in \mathbb{R}^r_+ \ | \ e^\top x \geq \sqrt{r-1} \|x\|_2 \} \; \subseteq \; \cone(H)$, 
\revise{where $\cone(H) = \{ x \ | \ x = Hy, y \geq 0 \}$ denotes the conical hull of the columns of $H$.} 

\noindent [SSC2]  There does not exist any orthogonal matrix $Q$ such that $\cone(H) \subseteq \cone(Q)$, except for permutation matrices. (An orthogonal matrix\index{matrix!orthogonal} $Q$ is a square matrix such that $Q^\top Q = I$)\revise{,} 
\end{definition}

SSC1 requires the columns of $H$ to contain the cone $\mathcal{C}$, which is tangent to every facet of the nonnegative orthant; see \Cref{fig_geoSSC}.   
\begin{figure*}[ht!] 
    \centering
        \includegraphics[width=\textwidth]{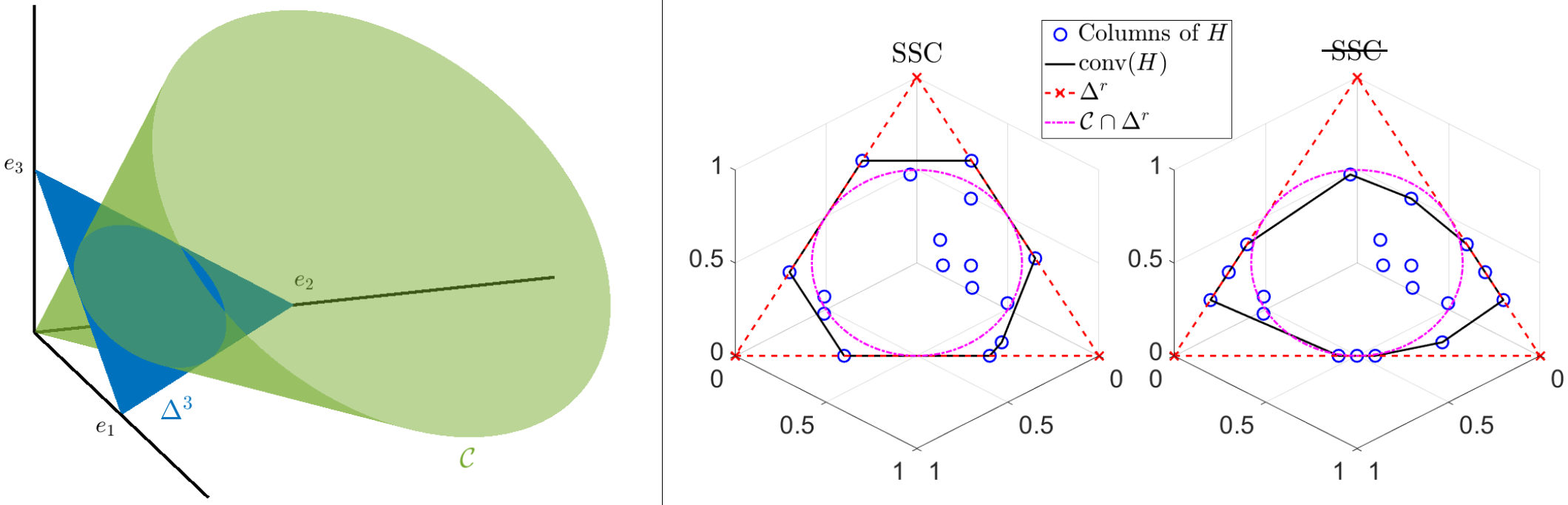} 
  \caption{Illustration of the \revise{SSC} in three dimensions. 
  On the left: the sets $\Delta^3$ and $\mathcal{C}$, they intersect at (0,0.5,0.5), (0.5,0,0.5), and (0.5,0.5,0).  
(This left figure is similar to~\cite[Fig. 2]{huang2013non} and we are grateful to the authors for providing us with the code to generate it.) 
On the right: examples of a matrix $H \in \mathbb{R}^{3 \times n}$ satisfying the SSC (left), and not  satisfying the SSC (right). The left (resp.\ right) figure is adapted from~\cite{huang2013non} (resp.~\cite{abdolali2021simplex}). 
    \label{fig_geoSSC}
    }
\end{figure*}    
Hence\revise{, satisfying SSC1} requires some degree of sparsity \revise{as $H$ needs to contain at least $r-1$ zeros per row~\cite[Th.~4.28]{gillis2020}}. 
SSC2 is a mild regularity condition which is typically satisfied when SSC1 is satisfied. For more discussions on the SSC, we refer the interested reader to \cite{xiao2019uniq} 
and \cite[Chapter 4.2.3]{gillis2020}, and the references therein. 
\revise{For SSMF, we have the following identifiability result.}  
\begin{theorem}\cite{fu2015blind, lin2015identifiability}  \label{th_identifminvolSSMF}
The minimum-volume SSMF factorization model,   
\revise{\[
\begin{split}
\min_{W \in \mathbb{R}^{m \times r}, H \in \mathbb{R}^{r \times n}} & \det(W^\top W) \\
\text{ such that } & X = WH \text{ and } H(:,j)\in\Delta^r \text{ for all } j, 
\end{split}
\]}
\noindent is identifiable
if the pair ${\revise{(W, H)}  \in \mathbb{R}^{m \times r} \times \mathbb{R}^{r \times n}}$
satisfies $\rank( \revise{W}) = r$ 
and $\revise{H}$ is sufficiently scattered.  
\end{theorem}
\revise{The quantity $\det(W^\top W)$ measures the volume of $\conv(W)$ within the column space of $W$. Note that this result has been generalized to the case where the columns of $H$ belong to a given polytope instead of the probability simplex; see~\cite{tatli2021generalized}.} 

\revise{In practice, because of noise and model 
misfit, 
SSMF optimization models need to balance the data fitting term which measures the discrepancy between $X$ and $WH$, and the volume regularization for $\conv(W)$. 
Typically a problem with objective function of the form 
\[
\|X-WH\|_F^2+\lambda\det(W^\top W), 
\] 
is solved. This  requires the tuning of the parameter $\lambda$, which is a nontrivial process~\cite{ang2019algorithms, Zhuang2019paramminvol}.  
}

\subsection{Nonnegative matrix factorization (NMF)} 
\label{sec_identif_nmf}

As opposed to SSMF, NMF decompositions can be identifiable without the use of additional requirements. \revise{The first identifiability result was proposed in \cite{donoho2004does}. Their conditions, based on separability, are quite strong. In the context of nonnegative source separation,~\cite{moussaoui2005non} proposed some necessary conditions for the uniqueness of the solution.}
One of the most relaxed sufficient condition for identifiability is based on the SSC.  
\begin{theorem}\cite[Theorem 4]{huang2013non} \label{th:uniqNMFSSC}
If $W^\top \in \mathbb{R}^{r \times m}$ and $H \in \mathbb{R}^{r \times n}$ are sufficiently scattered,  
then the Exact NMF $(W,H)$ of $X=WH$ of size $r = \rank(X)$ is essentially unique. 
\end{theorem}

In practice, it is not likely for both $W^\top$ and $H$ to satisfy the SSC. Typically $H$ will satisfy the SSC, as it is typically sparse. However, in many applications, $W^\top$ will not satisfy the SSC; in particular in applications where $W$ is not sparse, e.g., in hyperspectral umixing, recommender systems, or imaging. This is why regularized NMF models have been introduced, including sparse and minimum-volume NMF. 
We refer the interested reader to \cite[Chapter 4]{gillis2020} for more details.

\subsection{Bounded simplex-structured matrix factorization (BSSMF)} \label{sec_identifiability_bssmf}

A main motivation to introduce BSSMF is that it is identifiable under weaker conditions than NMF. 
We now state our main identifiability result for BSSMF, it is a consequence of the identifiability result of NMF and the following simple observation: $X = WH$ is a BSSMF for the interval $[a,b]$ implies that 
$be^\top - X =(be^\top - W)H$ and $X - ae^\top =(W - ae^\top) H$ are Exact NMF decompositions. 

\begin{theorem} \label{th:uniqueBSSMF} 
Let ${W \in \mathbb{R}^{m \times r}}$ and ${H \in \mathbb{R}^{r \times n}}$ satisfy ${W(:,k) \in [a,b]}$ for all $k$ for some $a \leq b$,   
${H \geq 0}$, and ${H^\top e = e}$. 
If $\binom{W - a e^\top}{b e^\top - W}^\top \in \mathbb{R}^{r \times 2m}$ and $H \in \mathbb{R}^{r \times n}$ are sufficiently scattered,  
then the BSSMF $(W,H)$ of $X=WH$ of size $r = \rank(X)$ for the interval $[a,b]$ is essentially unique. 
\end{theorem}
\begin{proof}
Let $(W,H)$ be a BSSMF of $X$ for the interval $[a,b]$. 
As in the proof of Lemma~\ref{lem:transfoBSSMF}, we have 
\[
X - a e^\top = WH - a e^\top = (W - a e^\top) H, 
\]
since \revise{$e^\top=e^\top H$}. 
This implies that $(W- a e^\top,H)$ is an Exact NMF of $X - a e^\top$, since $W- a e^\top$ and $H$ are nonnegative. 
Similarly, we have 
\[
b e^\top - X 
= b e^\top - WH 
= (b e^\top - W)H,  
\]
which implies that $(b e^\top - W, H)$ is an Exact NMF of $b e^\top - X$, since  $b e^\top - W \geq 0$. 
Therefore, we have the Exact NMF  
\[
\left( 
\begin{array}{c}
X - a e^\top \\ 
b e^\top - X 
\end{array}
\right) = 
\left( 
\begin{array}{c}
W - a e^\top \\ 
b e^\top - W  
\end{array}
\right) H.  
\] 
By Theorem~\ref{th:uniqNMFSSC}, this Exact NMF is unique if $\binom{W - a e^\top}{b e^\top - W}^\top$ and $H$ satisfy the SSC. 
This proves the result: in fact, the derivations above hold for any BSSMF of $X$. 
Hence, if $(W,H)$ was not an essentially unique BSSMF of $X$, there would exist another Exact NMF of $\binom{W - a e^\top}{b e^\top - W}^\top$, not obtained by permutation and scaling of $
\left( \binom{
W - a e^\top}{
b e^\top - W  
} , H \right)$, 
a contradiction. 
\end{proof}

\revise{Let us note that $W-ae^\top$ and $H$ being SSC, or $be^\top-W$ and $H$ being SSC, are also sufficient conditions for identifiability. These conditions are stronger, as $W-ae^\top$ being SSC or $be^\top-W$ being SSC implies that $\binom{W - a e^\top}{b e^\top - W}^\top$ is SSC. However, $\binom{W - a e^\top}{b e^\top - W}^\top$ does not imply that $W-ae^\top$ or $be^\top-W$ is SSC.} The condition that $\binom{W - a e^\top}{b e^\top - W}^\top$ is SSC is much weaker than requiring $W^\top$ to be SSC in NMF. 
\revise{In fact, in NMF, $W^\top$ is SSC requires that it contains 
zero entries (at least $r-1$ per row \cite[Th.~4.28]{gillis2020}; this can also be seen on the right of \Cref{fig_geoSSC} in the case $r=3$). 
Since the SSC is only defined for nonnegative matrices and $W^\top$ contains zeros, $a$ has to be equal to the zero vector. 
In this case, $W^\top$ being SSC implies that $W^\top - e a^\top$ 
is SSC, and hence the corresponding BSSMF is identifiable. However, the reverse is not true. In fact, $\binom{W - a e^\top}{b e^\top - W}^\top$ being SSC means that sufficiently many values in $W$ are equal to its minimum and maximum bounds in $a$ and $b$.} 
For example, in recommender systems, with $W(i,j) \in [1,5]$ for all $(i,j)$, many entries of $W$ are expected to be equal to 1 or to 5 (the minimum and maximum ratings), so that $\binom{W - a e^\top}{b e^\top - W}^\top$ will \revise{contain many zero entries}, and hence likely to satisfy the SSC\revise{~\cite{fu2018identifiability}}. On the other hand, $W$ is positive, and hence \revise{it cannot be part of} an essentially unique Exact NMF.

Let us illustrate the difference between NMF and BSSMF on a simple example. 
\begin{example}[Non-unique NMF vs.\ unique BSSMF] 
\label{example:1}
Let $\omega \in [0,1)$ and let 
\[
A_{\omega}  
= 
\left( \begin{array}{cccccc}
\omega & 1 & 1 & \omega & 0 & 0 \\ 
1 & \omega & 0 & 0 & \omega & 1 \\ 
0 & 0 & \omega & 1 & 1 & \omega \\ 
\end{array}
\right). 
\]
For $\omega < 0.5$, $A_{\omega}$ satisfies the SSC, while it does not for $\omega \leq 0.5$; see 
\cite[Example~3]{laurberg2008theorems}, 
\cite[Example~2]{huang2013non}, 
\cite[Example~4.16]{gillis2020}. 
Let us take 
\[
 H = 3 A_{1/3}  = 
\left( \begin{array}{cccccc}
1 & 3 & 3 & 1 & 0 & 0 \\ 
3 & 1 & 0 & 0 & 1 & 3 \\ 
0 & 0 & 1 & 3  & 3 & 1 \\ 
\end{array}
\right), 
\] 
which satisfies the SSC, and    
\[
W^\top = 3 A_{2/3} = 
\left( \begin{array}{cccccc} 
  2 &   3 &   3 &   2 &   0 &   0 \\ 
  3 &   2 &   0 &   0 &   2 &   3 \\ 
  0 &   0 &   2 &   3 &   3 &   2 \\ 
\end{array} \right), 
\]
which does not satisfy the SSC, but has some degree of sparsity. 
The NMF of 
\[
X 
= WH 
= \left( \begin{array}{cccccc} 
 11 &   9 &   6 &   2 &   3 &   9 \\ 
  9 &  11 &   9 &   3 &   2 &   6 \\ 
  3 &   9 &  11 &   9 &   6 &   2 \\ 
  2 &   6 &   9 &  11 &   9 &   3 \\ 
  6 &   2 &   3 &   9 &  11 &   9 \\ 
  9 &   3 &   2 &   6 &   9 &  11 \\ 
\end{array} \right)
\]
is not essentially unique. 
For example, 
\[
X = 
\left( \begin{array}{ccc} 
  0 &   3 &   1 \\ 
  1 &   3 &   0 \\ 
  3 &   1 &   0 \\ 
  3 &   0 &   1 \\ 
  1 &   0 &   3 \\ 
  0 &   1 &   3 \\ 
\end{array} \right) 
\left( \begin{array}{cccccc} 
  0 &   2 &   3 &   3 &   2 &   0 \\ 
  3 &   3 &   2 &   0 &   0 &   2 \\ 
  2 &   0 &   0 &   2 &   3 &   3 \\ 
\end{array} \right)
\]
is another decomposition which cannot be obtained as a scaling and permutation of $(W,H)$.  

However, the BSSMF of $X$ is unique, taking $a_i = 0$ and $b_i = 3$ for all $i$. In fact, $(3-W)^\top$ 
satisfies the SSC, as it is equal to  $3A_{1/3}$, up to permutation of its columns: 
\[
\begin{split}
3 - W^\top  
&= 
\left( \begin{array}{cccccc} 
  1 &   0 &   0 &   1 &   3 &   3 \\ 
  0 &   1 &   3 &   3 &   1 &   0 \\ 
  3 &   3 &   1 &   0 &   0 &   1 \\ 
\end{array} \right)\\
&= 
3 A_{1/3}(:, [4, 5, 6, 1, 2, 3] ).  
\end{split}
\]
Therefore, by Theorem~\ref{th:uniqueBSSMF}, the BSSMF of $X$ is unique.  
\end{example}


\paragraph{Scaling ambiguity}
BSSMF is in fact more than essentially unique in the sense of~\Cref{def:identifiability}. 
In fact, the scaling ambiguity can be removed because of $H$ being simplex structured, as shown in the following lemma. 
\begin{lemma}
	Let $H\in\mathbb{R}^{r\times n}$ such that $e^\top H=e^\top$ and $\rank(H)=r$. 
	Let $D \in \mathbb{R}^{r \times r}$ be a diagonal matrix, 
	and let $H'=DH$ be a scaling of the rows of $H$, and such that $e^\top H'=e^\top$. 
	Then $D$ must be the identity matrix, that is, $D=I$. 
\end{lemma}
\begin{proof} 
Let us denote $H^\dagger \in \mathbb{R}^{n\times r}$ the right inverse of $H$, which exists and is unique since $\rank(H) = r$, so that $H H^\dagger = I$. We have 
\begin{equation*}
\begin{split}
& e^\top H'= e^\top DH = e^\top \\
\Rightarrow \quad & e^\top DHH^\dagger=e^\top H^\dagger = e^\top \\
& \quad \quad \text{ since } e^\top H^\dagger=e^\top HH^\dagger=e^\top \\
\Rightarrow \quad & e^\top D = e^\top 
\quad \Rightarrow \quad D = I. 
\end{split}
\end{equation*}
Note that this lemma does not require $H$, $H'$ and $D$ to be nonnegative. 
\end{proof}

\paragraph{Geometric interpretation of BSSMF}

Solving BSSMF is equivalent to finding a polytope with $r$ vertices 
within the \revise{hyperrectangle} defined by \revise{$[a,b]$} 
 that reconstructs as well as possible the data points. The fact that BSSMF is constrained within a \revise{hyperrectangle} makes BSSMF more constrained \revise{than NMF}, and hence more likely to be  essentially unique. 
 This will be illustrated empirically in \Cref{sec_ident}. 
 Let us provide a toy example to better understand the distinction between NMF and BSSMF, namely let us use Example~\ref{example:1} with $W=\frac{3}{10}A_{2/3}$ and $H=\frac{2}{3}A_{1/2}$ so that $X=WH$ is column stochastic. 
 \Cref{fig:geobssmf} represents \revise{the feasible regions of} NMF and BSSMF 
 \revise{for the hypercube $[a,b] = 
 [0,\frac{3}{10}]^3$} in a two-dimensional space within the affine hull of $W$;  see~\cite{gillis2020} for the details on how to construct such a representation.  
 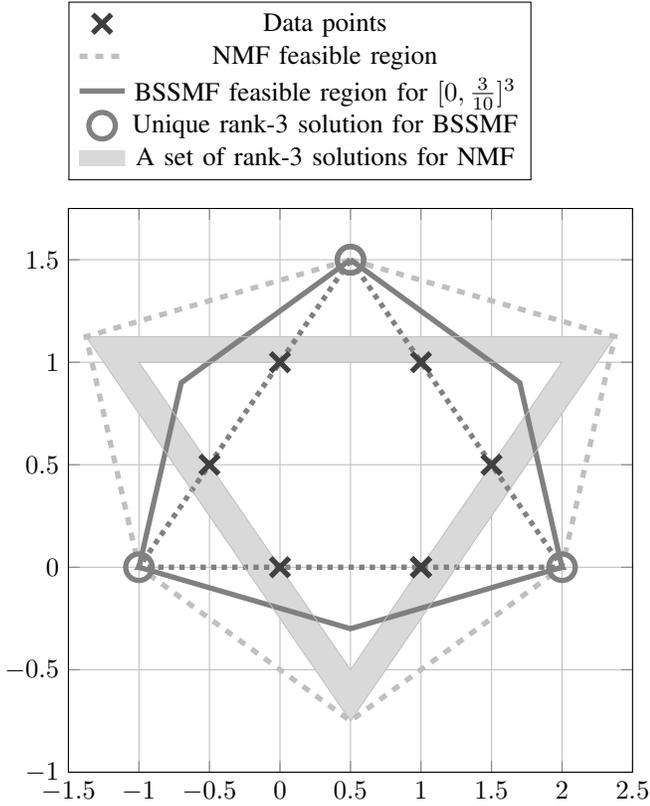
\begin{figure}[htb!]
%
%
\begin{tikzpicture}

\begin{axis}[%
width=7.5cm,
height=7.5cm,
at={(0cm,0cm)},
scale only axis,
xmin=-1.5,
xmax=2.5,
ymin=-1,
ymax=1.75,
xmajorgrids,
ymajorgrids,
legend style={at={(0,1.05)},anchor=south west}
]
\addplot [color=darkgray,line width=2.0pt,only marks,mark size=5pt, mark=x]
  table[row sep=crcr]{%
1	0\\
1.5	0.5\\
1	1\\
0	1\\
-0.5	0.5\\
0   0\\
1	0\\
};
\addlegendentry{Data points}

\addplot [color=lightgray, dashed, line width=2.0pt, mark size=7pt]
  table[row sep=crcr]{%
0.5 	1.5\\
2.375	1.125\\
2	0\\
0.5	    -0.75\\
-1	    0\\
-1.375	1.125\\
0.5 	1.5\\
};
\addlegendentry{NMF feasible region}

\addplot [color=gray, line width=2.0pt, mark size=7.0pt]
  table[row sep=crcr]{%
0.5 	1.5\\
1.7	0.9\\
2	0\\
0.5	    -0.3\\
-1	    0\\
-0.7	0.9\\
0.5 	1.5\\
};
\addlegendentry{BSSMF feasible region for $[0,\frac{3}{10}]^3$}

\addplot [color=gray, dotted, line width=2.0pt, forget plot]
  table[row sep=crcr]{%
0.5 	1.5\\
2	0\\
-1	    0\\
0.5 	1.5\\
};
\addplot [only marks, color=gray, line width=2.0pt, mark size=5.0pt, mark=o]
  table[row sep=crcr]{%
0.5 	1.5\\
2	0\\
-1	    0\\
0.5 	1.5\\
};
\addlegendentry{Unique rank-3 solution for BSSMF}

\addplot [name path=A1, color=lightgray, line width=0pt, forget plot]
  table[row sep=crcr]{%
2.375	1.125\\
0.5	    -0.75\\
-1.375	1.125\\
2.375	1.125\\
};
\addplot [name path=A2, color=lightgray, line width=0pt, forget plot]
  table[row sep=crcr]{%
2   1\\
0.5 -0.5\\
-1  1\\
2   1\\
};
\addplot[gray!30] fill between[of=A1 and A2];
\addlegendentry{A set of rank-3 solutions for NMF}
\end{axis}
\end{tikzpicture}
    \caption{Geometric interpretation of BSSMF~\revise{for Example~\ref{example:1}}. Any triangle in the gray filled area containing the data points is a rank-3 solution for NMF. On the contrary, there is a unique rank-3 solution for BSSMF since there is a unique triangle containing the data points in the BSSMF feasible set.}
    \label{fig:geobssmf}
\end{figure} 
 For this rank-3 factorization problem, solving NMF and BSSMF is equivalent to finding a triangle nested between the convex hull of the data points and the corresponding feasible region. 
 BSSMF has a unique solution, that is, there is a unique triangle between the data points and the BSSMF feasible region. On the other hand,  NMF is not identifiable: for example, any triangle within  the gray area containing the data points is a solution. 

In summary, for the BSSMF of $X = WH$ to be essentially unique, $W$ must contain sufficiently many entries equal to the lower and upper bounds, while $H$ must be sufficiently sparse. 

\paragraph{\revise{Choice of $a$ and $b$}} \label{para:choiceab} In practice, if $a$ and $b$ are unknown, it may be beneficial to choose them such that as many entries of $X$  are equal to the lower and upper bounds, \revise{and hence BSSMF is more likely to be identifiable. 
Let us denote $\tilde{a}_i = \min_j X(i,j)$ and $\tilde{b}_i = \max_j X(i,j)$ for all $i$, 
and let $X = WH$ be a BSSMF for the hyperrectangle $[a,b]$. 
We have $\tilde{a} \geq a$ and $\tilde{b} \leq b$ 
since $H(:,j) \in \Delta^r$ for all $j$. Hence, without any prior information, 
it makes sense to use a BSSMF with interval $[\tilde{a},\tilde{b}]$ which is contained in $[a,b]$.} 


\begin{remark}
    Interestingly, \revise{as shown in Lemma~\ref{lem:transfoBSSMF} below}, in the exact case, that is, when $X = WH$, we can assume w.l.o.g.\ that $[a_i, b_i] = [0,1]$ for all $i$ in BSSMF.   
\begin{lemma} \label{lem:transfoBSSMF} 
Let $a \in \mathbb{R}^m$ and $b \in \mathbb{R}^m$ 
be such that $a_i < b_i$ for all $i$. 
The matrix $X = WH$ admits a BSSMF for the interval 
$[a, b]$ if and only if the matrix 
$\frac{[X - a e^\top]}{[(b-a) e^\top]}$ admits a BSSMF for the interval $[0, 1]^m$, 
where $\frac{[\cdot]}{[\cdot]}$ is the component-wise division of two matrices of the same size. 
\end{lemma}
\begin{proof} 
Let us show the direction $\Rightarrow$, the other is obtained exactly in the same way. 
Let the matrix  $X = WH$ admit a BSSMF for the interval 
$[a, b]$. We have 
\[
X - ae^\top 
= WH - a e^\top 
= (W - a e^\top) H, 
\]
since $e^\top H = e^\top$, as $H$ is column stochastic. This shows that $X' = X - ae^\top$ admits a BSSMF for the interval $[0,b-a]$ since $W' = (W - a e^\top) \in [0,b-a]$. For simplicity, let us denote $c = b-a > 0$. We have $X' = W' H$, while 
\[
\frac{[X-a e^\top]}{[(b-a) e^\top]} 
= 
\frac{[X']}{[c e^\top]} 
= 
\frac{[W' H]}{[c e^\top]} 
= 
\frac{[W']}{[c e^\top]}  H , 
\]
because $H$ is column stochastic. In fact, for all $i,j$, 
\[
\begin{split}
\frac{[W' H]_{i,j}}{[c e^\top]_{i,j}} 
& = 
 \frac{\sum_{k} W'(k,i) H(k,j)]_{i,j}}{c_i} \\
& = \sum_{k} \frac{W'(k,i)}{c_i} H(k,j) \\
& = \left( \frac{[W']}{[c e^\top]} H  \right)_{i,j}.     
\end{split}
\] 
Hence $\frac{[X-a e^\top]}{[(b-a) e^\top]}$ admits a BSSMF for the interval $[0,1]^m$ since $H$ is column stochastic, 
and all columns of 
${\frac{[W']}{[c e^\top]} = \frac{[W - a e^\top]}{[(b-a) e^\top]}}$ belong to $[0,1]^m$. 
\end{proof}
\end{remark}

\begin{remark}[What if $a_i=b_i$ for some $i$?]  
Lemma~\ref{lem:transfoBSSMF} does not cover the case $a_i=b_i$ for some $i$. In that case, we have $W(i,:) = a_i = b_i$ and therefore 
${X(i,:) = W(i,:) H = a_i e^\top = b_i e^\top}$. 
This is not an interesting situation, and rows of $X$ with identical entries can be removed. In fact, after the transformation $X - ae^\top$, these rows are identically zero. 
\end{remark}

Lemma~\ref{lem:transfoBSSMF} highlights another interesting property of BSSMF: as opposed to NMF, it is invariant to translations of the entries of the input matrix, given that $a$ and $b$ are translated accordingly. 
For example, in recommender systems \revise{datasets} such as Netflix and MovieLens, $X(i,j) \in \{1,2,3,4,5\}$ for all $i,j$. 
Changing the scale, say to $\{0,1,2,3,4\}$, 
does not change the interpretation of the data, 
but will typically impact the NMF solution significantly\footnote{In fact, for NMF, it would make more sense to work on the \revise{datasets} translated to $[0,4]$, 
as it would potentially allow it to be identifiable: zeros in $X$ imply zeros in $W$ and $H$, which are therefore more likely to satisfy the SSC.}, 
while the BSSMF solution will be unchanged, if the interval is translated from $[1,5]$ to $[0,4]$ since $H$ is invariant by translation on $X$. This property is in fact coming from SSMF.

\section{Numerical experiments} \label{sec_numexp}

The goal of this section is to  highlight the motivation points mentioned in \Cref{sec_motiv} on real data sets. All experiments are run on a PC with an Intel(R) Core(TM) i7-9750H CPU @ 2.60GHz and 16GiB RAM. Let us recall that in order to retrieve NMF from \cref{alg:BSSMF}, the bounds need to be set to $(a,b)=(0,+\infty)$ and the projection step on the probability simplex in line~\ref{alg:BSSMF:line:proj} should \ovt{be replaced by a projection on the nonnegative orthant}. Hence, in our experiments, both NMF and BSSMF are solved with the same code implementation. 
\subsection{Interpretability} \label{sec_inter}

When applied on a pixel-by-image matrix, NMF allows to automatically extract common features among a set of images. For example, if each row of $X$ is a vectorized facial image, the rows of $W$ will correspond to facial features~\cite{lee1999learning}. 

Let us compare NMF with BSSMF on the widely \revise{used} MNIST handwritten digits dataset ($60,000$ images, $28\times28$ pixels) \cite{lecun1998gradient}. Each column of $X$ is a vectorized handwritten digit. For BSSMF to make more sense, we preprocess $X$ so that the intensities of the pixels in each digit belong to the interval $[0,1]$ (first remove from $X(:,j)$ its minimum entry, then divide by the maximum entry minus the minimum entry).

Let us take a toy example with $n=500$ randomly selected digits and $r=10$, in order to visualize the natural interpretability of BSSMF. \revise{The choice of $n$ is made solely for computational time considerations. For larger $n$, \Cref{fig:Wb} might change but we will not lose interpretability.}  
\begin{figure}[htb!]
    \begin{subfigure}{0.49\textwidth}
        \includegraphics[width=\textwidth]{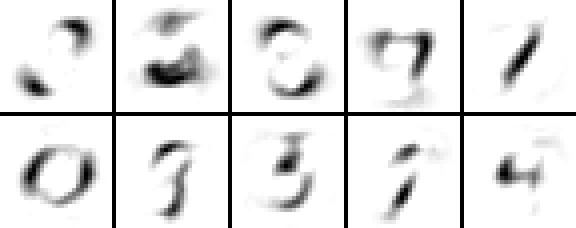}
        \caption{NMF}
        \label{fig:Wf}
    \end{subfigure}
    \begin{subfigure}{0.49\textwidth}
        \includegraphics[width=\textwidth]{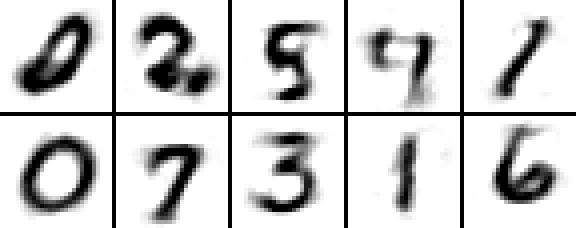}
        \caption{BSSMF}
        \label{fig:Wb}
    \end{subfigure}
    \caption{\revise{Reshaped columns of the} basis matrix $W$ for $r=10$ for MNIST with 500 digits.}
    \label{fig:Ws}
\end{figure} 
\Cref{fig:Wf} shows the features learned by NMF  which look like parts of digits. On the other hand, the features learned by BSSMF in \Cref{fig:Wb} look mostly like real digits, because of the bound constraint and the \revise{simplex structure}. 
%
\revise{In fact, as it is well known~\cite{lee1999learning} that NMF learns part-based representations, in this case, parts of digits. 
In other words, the columns of $W$ in NMF identify subset of pixels that are activated simultaneously in as many images as possible. 
Now, by the scaling degree of freedom, assume w.l.o.g.\ that $W(:,j) \in [0,1]^m$ for all $j$ in NMF. Since the columns of $W$ are parts of digits, each digits will have to use several of these parts, with an intensity close to one, so that $H$ will be far from being column stochastic. 
BSSMF, with the simplex constraint on $H$ and the bound constraints on $W$, therefore cannot  learn such a part-based representation. 
This is the reason why BSSMF learns more global features that, added on top of each other, reconstruct the digits. As it is shown in the MNIST experiment, these features look like digits themselves. 
Interestingly, if we progressively increase the upper bound, we would see that BSSMF progressively learns parts of digits, like NMF (using a lower bound of zero, that is, BSSMF with $[0,u]^m$ with $u \geq 1$). 
This is an indirect way of balancing the sparsity between $W$ and $H$. The larger the upper bound, the more relaxed is BSSMF and hence the sparser $W$ will be (given that the lower bound is 0).}   
%
\revise{In~\Cref{fig:Wb}, w}e distinguish numbers (like 7, 3 and 6). From a clustering point of view, this is of much interest because a column of $H$ which is near a ray of the probability simplex can directly be associated with the corresponding digit from~$W$. In this toy example, due to $r$ being small, 
an 8 cannot be seen. Nonetheless, an eight can be reconstructed as the weighted sum of the representations of a 5, a 3 and an italic 1; see \Cref{fig:eight} for an example.   
\begin{figure}[htb!]
    \centering
    \scalebox{0.68}{\input{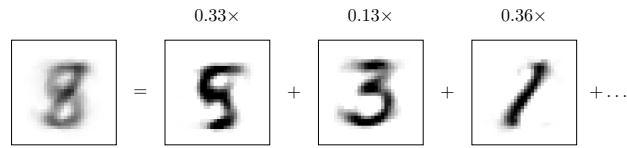}}
    \caption{Decomposition of an eight by BSSMF with $r=10$.}
    \label{fig:eight}
\end{figure} 
Note that since BSSMF is more constrained than NMF, its reconstruction error might be larger than that of NMF. 
For our example ($r=10$), BSSMF has  relative error $\|X-WH\|_F/\|X\|_F$ of $61.56\%$, and  NMF of $59.04\%$. \revise{This is not always a drawback. In some applications, due to the presence of noise, although the reconstruction error of BSSMF is larger than that of NMF, the accuracy of the estimated factors $W$ and $H$ could be better, because it uses the prior information and is less prone to overfitting and less sensitive to outliers. 
See also the discussion in \Cref{sec_robust}   
 where NMF has a lower RMSE than BSSMF on the training set, but a larger RMSE than BSSMF on the test set.} 
\ngi{Note that we also compute NMFs using \cref{alg:BSSMF} where the projections are performed on the nonnegative orthant, instead of on the bounded set for $W$ and on the probability simplex for~$H$.} \ovt{The stopping criteria in line~\ref{alg:BSSMF:line:outerloop}, \ref{alg:BSSMF:line:Winnerloop} and \ref{alg:BSSMF:line:Hinnerloop} of \cref{alg:BSSMF} are a maximum number of iterations equal to 500, 20 and 20, respectively, for both algorithms.}

\subsection{Identifiability} \label{sec_ident}
As it is NP-hard to check the SSC~\cite{huang2013non}, we perform experiment\revise{s} \revise{on MNIST and synthetic data} where only a necessary condition for SSC1 is verified, 
namely \cite[Alg.~4.2]{gillis2020}. 

\paragraph{MNIST dataset} \revise{On MNIST, t}o see when $H$ satisfies this condition, we first vary $n$ from $100$ to $300$ for $m$ fixed (=$28$$\times$$28$). For $W^\top$, we fix $n$ to $300$, and downscale the resolution $m$ from $28$$\times$$28$ to $12$$\times$$12$ with a linear interpolation (\texttt{imresize3} in MATLAB), and the rank $r$ is varied from $12$ to $30$. 
Recall that both factors need to satisfy the SSC to correspond to an essentially unique factorization. In \Cref{fig:ssc_nmf}, we see that $W^\top$ of NMF often satisfies the necessary condition. 
This is due to NMF learning ``parts'' of objects~\cite{lee1999learning}, which are sparse by nature, and sparse matrices are more likely to satisfy \revise{the SSC} (\cref{def:ssc}).  
On the contrary, even for a relatively large $n$, $H$ is too dense to satisfy the necessary condition. For $r\geq30$, the factor $H$ generated by NMF never satisfies the condition.
Meanwhile, in \Cref{fig:ssc_bssmf} we see that $H$ of BSSMF always satisfies the condition when $n \geq225$ for $r=30$ and more generally, if $n$ and $m$ are large enough, both $H$ and $\binom{W}{\revise{J} - W}^\top$ satisfy the necessary condition. This substantiates that BSSMF provides  essentially unique factorizations more often than NMF does. 

\begin{figure}[htb!]
  \centering
  \begin{subfigure}{0.51\textwidth}
      \def\step{0.28}
\begin{tikzpicture}
\coordinate (grid_start) at (-5*\step , -4*\step);
\coordinate (grid_end) at (4*\step , 5*\step);
\node (titleH) [yshift=5*\step cm+0.5cm] {$H$};
\foreach \x in {1, ..., 5}
  \draw (-5*\step + 2*\x*\step -1.5*\step ,-4*\step) -- (-5*\step + 2*\x*\step-1.5*\step ,-4*\step cm -1pt) node[anchor=north] {$\pgfmathparse{10+(\x-1)*10}\pgfmathprintnumber{\pgfmathresult}$};
\node (xaxis) [yshift=-5*\step cm-0.3 cm] {rank $r$};
\foreach \x in {1, ..., 5}
  \draw (-5*\step , -4*\step + 2*\x*\step -1.5*\step) -- (-5*\step cm -1pt, -4*\step + 2*\x*\step-1.5*\step) node[anchor=east] {$\pgfmathparse{300-(\x-1)*50}\pgfmathprintnumber{\pgfmathresult}$};
\node (yaxis) [xshift=-6*\step cm-0.7 cm,rotate=90] {number of samples $n$};

\foreach \y [count=\n] in {
    {100,80,50,10,0,0,0,0,0},
    {100,100,70,0,0,0,0,0,0},
    {100,100,70,50,10,0,0,0,0},
    {100,100,90,50,0,0,0,0,0},
    {100,100,90,40,0,0,0,0,0},
    {100,100,90,30,10,0,0,0,0},
    {100,100,90,20,0,0,0,0,0},
    {100,100,100,20,10,0,0,0,0},
    {100,100,100,20,0,0,0,0,0},
} {
  \foreach \x [count=\m] in \y {
    \fill [white!\x!black] (-6*\step + \m*\step,6*\step - \n*\step) rectangle (4*\step cm,-4*\step cm);
  }
}
\draw[step=\step,gray,very thin] (grid_start) grid (grid_end);
\end{tikzpicture}
\hspace{-0.3cm}
\begin{tikzpicture}
\coordinate (grid_start) at (-5*\step , -4*\step);
\coordinate (grid_end) at (4*\step , 5*\step);
\node (titleW) [yshift=5*\step cm+0.5cm] {$W^\top$};
\foreach \x in {1, ..., 5}
  \draw (-5*\step + 2*\x*\step -1.5*\step ,-4*\step) -- (-5*\step + 2*\x*\step-1.5*\step ,-4*\step cm -1pt) node[anchor=north] {$\pgfmathparse{14+(\x-1)*4}\pgfmathprintnumber{\pgfmathresult}$};
\node (xaxis) [yshift=-5*\step cm-0.3 cm] {rank $r$};
\foreach \x in {1, ..., 5}
  \draw (-5*\step , -4*\step + 2*\x*\step -1.5*\step) -- (-5*\step cm -1pt, -4*\step + 2*\x*\step-1.5*\step) node[anchor=east] {$\pgfmathparse{28-(\x-1)*4}\pgfmathprintnumber{\pgfmathresult}$};
\node (yaxis) [xshift=-6*\step cm-0.5 cm,rotate=90] {resolution $\sqrt{m}$};

\foreach \y [count=\n] in {
    {80,80,60,20,10,10,0,0,0},
    {100,60,70,40,20,10,0,0,0},
    {100,100,90,80,80,60,40,10,10},
    {100,100,80,80,100,90,50,30,20},
    {100,100,90,90,100,90,70,70,40},
    {100,100,100,100,100,100,90,70,90},
    {100,100,100,100,100,100,90,90,100},
    {100,100,100,100,100,100,90,100,90},
    {100,100,100,100,100,100,100,100,100},
} {
  \foreach \x [count=\m] in \y {
    \fill [white!\x!black] (-6*\step + \m*\step,6*\step - \n*\step) rectangle (4*\step cm,-4*\step cm);
  }
}
\draw[step=\step,gray,very thin] (grid_start) grid (grid_end);
\node (colormap) [yshift = 0.5*\step cm, xshift = 5*\step cm + 0.3 cm] {\pgfplotscolorbardrawstandalone[ 
    colormap={blackwhite}{color=(black) color=(white)},
    colorbar,
    point meta min=0,
    point meta max=1,
    colorbar style={
        width = \step cm,
        height=9*\step cm,
        ytick={0,0.2,0.4,0.6,0.8,1}}]};   
\end{tikzpicture}
      \caption{NMF}
      \label{fig:ssc_nmf}
  \end{subfigure}
  \hspace{-0.8cm}
  \begin{subfigure}{0.52\textwidth}
      \def\step{0.28}
\begin{tikzpicture}
\coordinate (grid_start) at (-5*\step , -4*\step);
\coordinate (grid_end) at (4*\step , 5*\step);
\node (titleH) [yshift=5*\step cm+0.5cm] {$H$};
\foreach \x in {1, ..., 5}
  \draw (-5*\step + 2*\x*\step -1.5*\step ,-4*\step) -- (-5*\step + 2*\x*\step-1.5*\step ,-4*\step cm -1pt) node[anchor=north] {$\pgfmathparse{10+(\x-1)*10}\pgfmathprintnumber{\pgfmathresult}$};
\node (xaxis) [yshift=-5*\step cm-0.3 cm] {rank $r$};
\foreach \x in {1, ..., 5}
  \draw (-5*\step , -4*\step + 2*\x*\step -1.5*\step) -- (-5*\step cm -1pt, -4*\step + 2*\x*\step-1.5*\step) node[anchor=east] {$\pgfmathparse{300-(\x-1)*50}\pgfmathprintnumber{\pgfmathresult}$};
\node (yaxis) [xshift=-6*\step cm-0.7 cm,rotate=90] {number of samples $n$};

\foreach \y [count=\n] in {
    {90,100,90,50,10,0,0,0,0},
    {100,90,80,70,30,0,0,0,0},
    {100,100,90,90,80,30,10,0,0},
    {100,100,100,90,90,50,10,0,0},
    {100,100,100,100,90,70,40,20,0},
    {100,100,100,100,100,80,80,50,0},
    {100,100,100,100,100,100,90,50,30},
    {100,100,100,100,100,90,80,80,60},
    {100,100,100,100,100,100,100,80,70},
} {
  \foreach \x [count=\m] in \y {
    \fill [white!\x!black] (-6*\step + \m*\step,6*\step - \n*\step) rectangle (4*\step cm,-4*\step cm);
  }
}
\draw[step=\step,gray,very thin] (grid_start) grid (grid_end);
\end{tikzpicture}
\hspace{-0.3cm}
\begin{tikzpicture}
\coordinate (grid_start) at (-5*\step , -4*\step);
\coordinate (grid_end) at (4*\step , 5*\step);
\node (titleW) [yshift=5*\step cm+0.5cm] {$\binom{W}{ee^\top - W}^\top$};
\foreach \x in {1, ..., 5}
  \draw (-5*\step + 2*\x*\step -1.5*\step ,-4*\step) -- (-5*\step + 2*\x*\step-1.5*\step ,-4*\step cm -1pt) node[anchor=north] {$\pgfmathparse{14+(\x-1)*4}\pgfmathprintnumber{\pgfmathresult}$};
\node (xaxis) [yshift=-5*\step cm-0.3 cm] {rank $r$};
\foreach \x in {1, ..., 5}
  \draw (-5*\step , -4*\step + 2*\x*\step -1.5*\step) -- (-5*\step cm -1pt, -4*\step + 2*\x*\step-1.5*\step) node[anchor=east] {$\pgfmathparse{28-(\x-1)*4}\pgfmathprintnumber{\pgfmathresult}$};
\node (yaxis) [xshift=-6*\step cm-0.5 cm,rotate=90] {resolution $\sqrt{m}$};

\foreach \y [count=\n] in {
    {70,20,0,0,0,0,0,0,0},
    {70,30,10,0,0,0,0,0,0},
    {100,80,70,20,20,0,0,0,0},
    {80,100,80,50,0,0,0,0,0},
    {70,100,90,50,70,30,0,0,0},
    {80,90,100,100,100,40,40,10,0},
    {90,100,100,100,90,70,40,30,0},
    {100,80,100,100,100,70,80,30,30},
    {100,90,100,100,100,100,100,100,80},
} {
  \foreach \x [count=\m] in \y {
    \fill [white!\x!black] (-6*\step + \m*\step,6*\step - \n*\step) rectangle (4*\step cm,-4*\step cm);
  }
}
\draw[step=\step,gray,very thin] (grid_start) grid (grid_end);
\node (colormap) [yshift = 0.5*\step cm, xshift = 5*\step cm + 0.3 cm] {\pgfplotscolorbardrawstandalone[ 
    colormap={blackwhite}{color=(black) color=(white)},
    colorbar,
    point meta min=0,
    point meta max=1,
    colorbar style={
        width = \step cm,
        height=9*\step cm,
        ytick={0,0.2,0.4,0.6,0.8,1}}]};   
\end{tikzpicture}
      \caption{BSSMF}
      \label{fig:ssc_bssmf}
  \end{subfigure}
  \caption{Ratio, over 10 runs, of the factors generated by NMF in \Cref{fig:ssc_nmf} and by BSSMF in \Cref{fig:ssc_bssmf} that satisfy the necessary condition for SSC1  (white squares indicate that all matrices meet the necessary condition, black squares that none do). }
  \label{fig:ssc_nec_cond}
\end{figure}
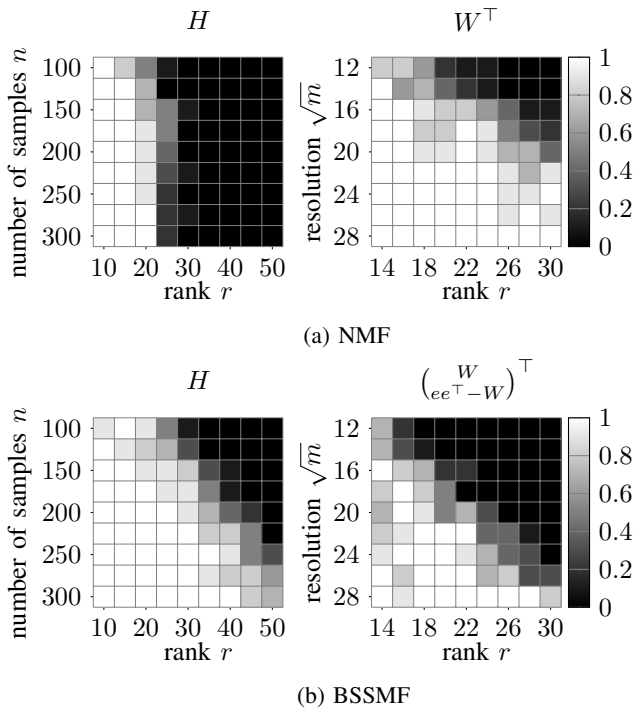

\revise{\paragraph{Synthetic datasets} 
Let us now perform an experiment to show how BSSMF is more likely than NMF to recover factors closer to the true ones, even when the sufficient conditions for identifiability are not satisfied. 
As there is no groundtruth for NMF and BSSMF on MNIST, we generate synthetic data as follows. Our synthetic datasets are of size $100\times100$, and their factorization rank is $10$. 
The matrix $H$ is generated randomly with values uniformly distributed between zero and one, and we randomly set $30\%$ of the values to zero. This allows us to ensure that $H$ satisfies the SSC. The reason behind ensuring that $H$ is SSC is that both 
NMF (\Cref{th:uniqNMFSSC}) and BSSMF (\Cref{th:uniqueBSSMF}) require that $H$ satisfies the SSC\footnote{In this experiment, because $n$ and $r$ are smaller, we could checked that the SSC is satisfied (not a necessary condition), using Gurobi (\url{https://www.gurobi.com/}), a global optimization software.}. 
As we want to emphasize on how likely it is to retrieve the true factors for NMF and BSSMF, we make sure that their common conditions for identifiability are satisfied. 
The matrix $W$ is also generated randomly with values uniformly distributed between zero and one, and we then set a percentage of $p_{0,1}$ of the entries to zero and one, with the same probability to be equal to zero or one. Hence, $p_{0,1}$ percent of the values in $W$ touches the lower and upper bounds in BSSMF. 
Finally, we let $X=WH$ to get our synthetic data. We solve NMF and BSSMF on $X$ using \cref{alg:BSSMF}. To assess the quality of the solutions, we report the average of the mean removed spectral angle (MRSA) between the columns of the true $W$ and the estimated $W$ (after an optimal permutation of the columns), as this is standard in the NMF literature. Given any two vectors $a$ and $b$, their MRSA is defined as 
\begin{equation*}
    \text{MRSA}(a,b) = \frac{100}{\pi} \text{arcos} 
    \left(\frac{ (a - \overline{a}e)^\top (b - \overline{b}e) }{\Vert a - \overline{a}e \Vert_{2}\Vert b - \overline{b} e\Vert_{2}}\right) \in \left[  0, 100\right], 
\end{equation*}
where $\overline{\cdot}$~is the average of the entries of a vector. 

\begin{figure}[htb!]
    \centering
    \pgfplotsset{boxplot circle legend/.style={
    legend image code/.code={
        \draw[#1,line width=0.7pt] (0cm,-0.1cm) rectangle (0.6cm,0.1cm)
        (0.6cm,0cm) -- (0.7cm,0cm) (0cm,0cm) -- (-0.1cm,0cm)
        (0.7cm,0.1cm) -- (0.7cm,-0.1cm) (-0.1cm,0.1cm) -- (-0.1cm,-0.1cm)
        (0.3cm,-0.1cm) -- (0.3cm,0.1cm);
        \filldraw[#1] (0.9cm,0cm) circle (2pt);
    },},
    boxplot halfcircle legend/.style={
    legend image code/.code={
        \draw[#1,line width=0.7pt] (0cm,-0.1cm) rectangle (0.6cm,0.1cm)
        (0.6cm,0cm) -- (0.7cm,0cm) (0cm,0cm) -- (-0.1cm,0cm)
        (0.7cm,0.1cm) -- (0.7cm,-0.1cm) (-0.1cm,0.1cm) -- (-0.1cm,-0.1cm)
        (0.3cm,-0.1cm) -- (0.3cm,0.1cm);
        \begin{scope}
            \clip (0.9cm,0cm) circle (2pt);
            \fill[#1] (0.9cm+2pt,2pt) rectangle ++(-4pt,-1.9pt);
        \end{scope}
        \draw[#1,line width=0.7pt] (0.9cm,0cm) circle (2pt);
    },},
    }
\begin{tikzpicture}
\begin{axis}[
    width=0.49\textwidth,
    y dir=reverse,
    yticklabels={0\%,5\%,10\%,15\%,20\%,25\%,30\%},
    ymin=0,ymax=7,
    y tick label as interval,
    y=1cm,
    ytick={0,1,2,...,7},
    xlabel={Average MRSA},ylabel={\% of values equal to 0 and 1},
    xmode=log,
    xmin=1.0e-8,xmax=3,
    xmajorgrids=true,
    grid style=dashed,
    boxplot={
        draw position={1/(2+1) + floor(\plotnumofactualtype/2+1e-3) + 1/(2+1)*mod(\plotnumofactualtype,2)},
        box extend=1/(2+2)},
    legend style={
        at={(0.05,0.95)},anchor=north west,},]
\addlegendimage{boxplot halfcircle legend=gray}
\addlegendimage{boxplot circle legend=black}
\foreach \col in {0,...,6}{
\addplot[boxplot,gray,mark=halfcircle*,line width=0.7pt] table[y index = \col]{data/MRSA_NMF.txt};
\addplot[boxplot,black,line width=0.7pt] table[y index = \col]{data/MRSA_BSSMF.txt};
};
\legend{NMF,BSSMF}
\end{axis}
\end{tikzpicture}
    \caption{Boxplots of the average MRSA between the true $W$ and the estimated $W$ by NMF and BSSMF for the hypercube $[0,1]^{100}$ over 20 trials, depending on the percentage, $p_{0,1}$, of values equal to 0 and 1 in the true $W$.}
    \label{fig:mrsa_nmf_vs_bssmf}
\end{figure}
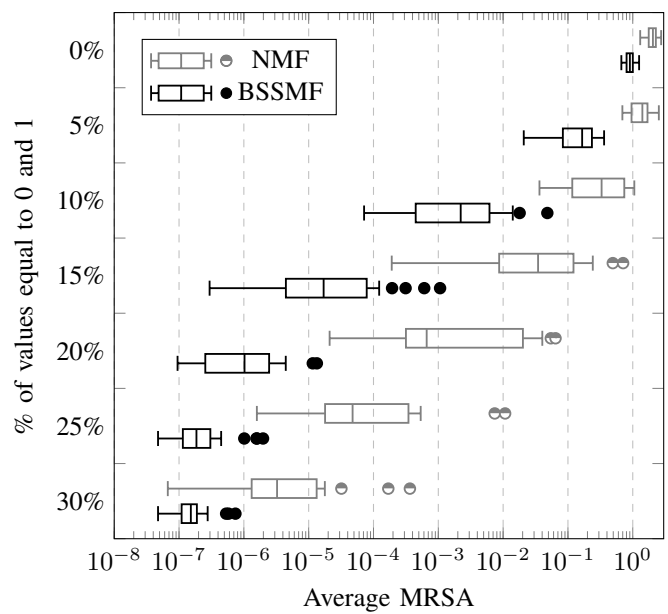  
We vary the percentage $p_{0,1}$ of values touching the lower and uppper bounds in $W$ (namely, 0 and 1) from $0\%$ to $30\%$ with a $5\%$ increment. For each value of $p_{0,1}$, the test is performed 20 times. 
Let us note that among the generated true $W$'s, 
between $p_{0,1}=0\%$ and $p_{0,1}=15\%$, $\binom{W}{J-W}^\top$ never satisfies the necessary conditions for SSC1. 
For $p_{0,1}=20\%$, 3 out of the 20 generated $\binom{W}{J-W}^\top$ satisfies the necessary conditions for SSC1, 
10 out of 20 for $p_{0,1}=25\%$, 
and 17 out of 20 for $p_{0,1}=30\%$. 
Let us also note that for all values of $p_{0,1}$ within the considered range, $W$ never satisfies the necessary conditions for SSC1. 
The distribution of the average MRSAs is reported in~\Cref{fig:mrsa_nmf_vs_bssmf}. Clearly, the MRSA is always smaller for BSSMF compared to NMF, even when the necessary conditions for SSC1 are not satisfied for $\binom{W}{J-W}^\top$; this is because the feasible set of BSMF is contained in that of NMF, and hence the generated factors are more likely to be closer to the groundtruth. This also illustrates that the conditions of \Cref{th:uniqueBSSMF} for the identifiability of BSSMF are only sufficient, since BSSMF finds solutions with MRSA close to machine epsilon when these conditions are not fulfilled.} 

\subsection{Robustness to overfitting} \label{sec_robust}

In this section we compare 
unconstrained matrix factorization (MF), 
NMF and BSSMF on the matrix completion problem; more precisely, on rating datasets for recommendation systems. Let $X$ be an item-by-user matrix and suppose that user $j$ has rated item $i$, that rating would be stored in $X_{i,j}$. The matrix $X$ is then highly incomplete since a user has typically only rated a few of the items. 
In this context, NMF looks for nonnegative factors $W$ and $H$ such that $M\circ X\approx M\circ (WH)$, where $M_{i,j}$ is equal to $1$ when user $j$ rated item $i$ and is equal to $0$ otherwise. A missing rating $X_{i,j}$ is then estimated by computing $W(i,:)H(:,j)$. Features learned by NMF on rating datasets tend to be parts of typical users. Yet, the nonnegative constraint on the factors hardly makes the features interpretable by a practitioner. Suppose that the rating a user can give is an integer between $1$ and $5$ like in many rating systems, NMF can learn features whose values may fall under the minimum rating $1$ or may exceed the maximum rating $5$. Consequently, the features cannot directly be interpreted as typical users. On the contrary, with BSSMF, the extracted features will directly be interpretable if the lower and upper bounds are set to the minimum and maximum ratings. On top of that, BSSMF is expected to be less sensitive to overfitting than NMF since its feasible set is more constrained.

This last point will be highlighted in the following experiment on the ml-1m dataset\footnote{\href{https://grouplens.org/datasets/movielens/1m/}{https://grouplens.org/datasets/movielens/1m/}}, which contains 1 million ratings from 6040 users on 3952 movies. As in~\cite{liang2018variational}, we split the data in two sets~: a training set and a test set. The test set contains 500 users. We also remove any movie that has been rated less than 5 times \revise{from both the training and test sets}. For the test set, 80\% of a user's ratings are considered as known. The remaining 20\% are kept for evaluation. During the training, we learn $W$ only on the training set. During the testing, the learned $W$ is used to predict those 20\% kept ratings of the test set by solving the $H$ part only on the 80\% known ratings. This simulates new users that were not taken into account during the training, but for whom we would still want to predict 
the ratings. 
The reported \revise{RMSEs} are computed on the 20\% kept ratings of the test set. In order to challenge the overfitting issue, we vary $r$ in $\{1,5,10,20,50,100\}$ for BSSMF, NMF and an unconstrained MF \ngi{which are all computed using \cref{alg:BSSMF}, where the projections onto the feasible sets are adapted accordingly (projection onto the nonnegative orthant for NMF, no projection for unconstrained MF)}. 
\ovt{The stopping criteria in line~\ref{alg:BSSMF:line:outerloop}, \ref{alg:BSSMF:line:Winnerloop} and \ref{alg:BSSMF:line:Hinnerloop} of \cref{alg:BSSMF} are a maximum number of iterations equal to 200, 1 and 1, respectively, for all algorithms.} 
The experiment is conducted on 10 random initializations and the average RMSEs are reported is \Cref{tab:rmse_ml-1m}. As expected, BSSMF and NMF are more robust to overfitting than unconstrained MF. Additionnaly, BSSMF is also clearly more robust to overfitting than NMF. Its worse RMSE is $0.89$ with $r=100$ (and it is still equal to $0.89$ with $r=200$), while, for NMF, the RMSE is $0.92$ when $r=100$ (which is worse than a rank-one factorization giving a RMSE of $0.91$). 
\begin{table}[htb!]
    \centering
    \begin{tabular}{r|lll}
        r & BSSMF & NMF & MF \\ \hline
        1 & $0.97 \pm 2\cdotp10^{-5}$ & $0.88 \pm 0.002$ & $0.91 \pm 5\cdotp10^{-6}$ \\
        5 & $0.87 \pm 0.001$ & $0.87 \pm 0.003$ & $0.87 \pm 0.003$ \\
        10 & $0.86 \pm 0.002$ & $0.87 \pm 0.001$ & $0.87 \pm 0.002$ \\
        20 & $0.87 \pm 0.002$ & $0.87 \pm 0.002$ & $0.88 \pm 0.002$ \\
        50 & $0.88 \pm 0.002$ & $0.90 \pm 0.004$ & $0.93 \pm 0.004$ \\
        100 & $0.89 \pm 0.003$ & $0.92 \pm 0.003$ & $0.99 \pm 0.004$
    \end{tabular}
    \caption{RMSE on the test set according to $r$, averaged \textpm{ standard deviation} on 10 runs on ml-1m}
    \label{tab:rmse_ml-1m}
\end{table}


 The same experiment is conducted on the ml-100k dataset\footnote{\href{https://grouplens.org/datasets/movielens/100k/}{https://grouplens.org/datasets/movielens/100k/}} which contains 100,000 ratings from 1,700 movies rated by 1,000 users. The test set contains 50 users. The results are reported in \Cref{tab:rmse_ml-100k}, and the observations are similar: BSSMF is significantly more robust to overfitting than NMF and unconstrained MF. 
\begin{table}[htb!]
    \centering
    \begin{tabular}{r|lll}
        r & BSSMF & NMF & MF \\ \hline
        1 & $0.98 \pm 1\cdotp10^{-4}$ & $0.91 \pm 3\cdotp10^{-5}$ & $0.91 \pm 5\cdotp10^{-5}$\\
        5 & $0.89 \pm 0.005$ & $0.89 \pm 0.01$ & $0.89 \pm 0.008$ \\
        10 & $0.90\pm 0.008$ & $0.90 \pm 0.009$ & $0.92 \pm 0.01$ \\
        20 & $0.91 \pm 0.01$ & $0.93 \pm 0.01$ & $0.97 \pm 0.02$ \\
        50 & $0.93 \pm 0.01$ & $0.97 \pm 0.01$ & $1.06 \pm 0.03$ \\
        100 & $0.94 \pm 0.01$ & $1.01 \pm 0.007$ & $1.13 \pm 0.02$
    \end{tabular}
    \caption{RMSE on the test set according to $r$, averaged \textpm{ standard deviation} on 10 runs on ml-100k}
    \label{tab:rmse_ml-100k}
\end{table} 
\vspace{-0.6cm}

\section{Conclusion} 

In this paper, we proposed a new factorization model, namely bounded simplex structured matrix factorization (BSSMF). 
Fitting this model retrieves interpretable factors: 
the learned basis features can be interpreted in the same way as the original data while the activations are nonnegative and sum to one, leading to a straightforward soft clustering interpretation. 
Instead of learning parts of objects as NMF, BSSMF learns objects that  can be used to explain the data through convex combinations. 
We have proposed a dedicated fast algorithm for BSSMF, and showed that, 
under mild conditions, BSSMF is essentially unique. 
 We also showed that the constraints in BSSMF make it robust to overfitting on rating datasets without adding any regularization term. 

 \revise{Further work include:  \begin{itemize}
     \item the use of BSSMF for other applications, 

     \item the generalization of BSSMF, e.g., as done in~\cite{tatli2021generalized} for SSMF where the feasible set for the columns of $H$, namely the probability simplex, is replaced by any polytope,  

     \item  the design of more efficient algorithms for BSSMF,  and 

     \item  the design of algorithms for other BSSMF models, e.g., with other data fitting terms such as the Kullback-Leibler divergence, as done recently in~\cite{leplat2021multiplicative} for SSMF with nonnegativity constraint on $W$. 
 \end{itemize} 
}

\revise{
\subsection*{Acknowledgment}   
We sincerely thank the Associate Editor and the reviewers for taking the time to carefully read the paper, and for the very detailed and insightful feedback that helped us improve our paper.
}

\bibliographystyle{plain}
\bibliography{BoundedSSMF}

\end{document}